\documentclass[10pt,journal,compsoc]{IEEEtran}
\ifCLASSOPTIONcompsoc
  \usepackage[nocompress]{cite}
\else
  \usepackage{cite}
\fi
\ifCLASSINFOpdf
\else
\fi
\usepackage{pifont}
\usepackage{bm}
\usepackage{amsmath}
\usepackage{graphicx}
\usepackage{subfigure}

\usepackage{algorithmic}
\usepackage{multirow}
\usepackage{array}
\usepackage{amssymb}
\usepackage{epsfig}
\usepackage{float}
\usepackage{stfloats}
\usepackage{multirow}
\usepackage{subfloat} 
\usepackage{booktabs}
\usepackage{amsthm}
\usepackage{caption} 
\usepackage{amsthm}
\usepackage{mathrsfs}
\usepackage{color}
\usepackage{xcolor}
\usepackage{bbding}
\newcommand{\rone}[1]{{\color{black} #1}}
\usepackage{url}

\newcommand{\etal}{\textit{et al.}} 

\usepackage[breaklinks,colorlinks,linkcolor=black,citecolor=black,urlcolor=black]{hyperref}

\newtheorem{definition}{Definition}
\newtheorem{proposition}{Proposition}

\begin{document}
\title{Scaling Spike-driven Transformer with Efficient Spike Firing Approximation Training}
\author{Man~Yao*, Xuerui~Qiu*, Tianxiang~Hu, Jiakui~Hu, Yuhong~Chou, Keyu~Tian, \\ Jianxing~Liao, Luziwei~Leng, Bo~Xu, and Guoqi~Li$^{\dag}$

\IEEEcompsocitemizethanks{\IEEEcompsocthanksitem M. Yao and X. Qiu and T. Hu, B.Xu and G.Li are with Institute of Automation, Chinese Academy of Sciences, Beijing, China.\protect
\IEEEcompsocthanksitem X. Qiu is also with School of Future Technology, University of Chinese Academy of Sciences, Beijing, China.\protect
\IEEEcompsocthanksitem J. Hu and K.Tian are with the Institute for Artificial Intelligence, Peking University, Beijing, China.\protect
\IEEEcompsocthanksitem Y. Chou is with Department of Computing, The Hong Kong Polytechnic University, HongKong, China.\protect
\IEEEcompsocthanksitem J. Liao and L. Leng are with Huawei ACS Lab, Shenzhen, GuangDong, China.\protect\\
* Equal contribution. \protect\\
$^{\dag}$ Corresponding author: Guoqi Li (E-mail:guoqi.li@ia.ac.cn).\protect}

}
\markboth{FOR REVIEW}%
{Shell \MakeLowercase{\textit{et al.}}: Bare Demo of IEEEtran.cls for Computer Society Journals
}

\IEEEtitleabstractindextext{%
\begin{abstract}
The ambition of brain-inspired Spiking Neural Networks (SNNs) is to become a low-power alternative to traditional Artificial Neural Networks (ANNs). This work addresses two major challenges in realizing this vision: the performance gap between SNNs and ANNs, and the high training costs of SNNs. We identify intrinsic flaws in spiking neurons caused by binary firing mechanisms and propose a Spike Firing Approximation (SFA) method using integer training and spike-driven inference. This optimizes the spike firing pattern of spiking neurons, enhancing efficient training, reducing power consumption, improving performance, enabling easier scaling, and better utilizing neuromorphic chips. We also develop an efficient spike-driven Transformer architecture and a spike-masked autoencoder to prevent performance degradation during SNN scaling. On ImageNet-1k, we achieve state-of-the-art top-1 accuracy of 78.5\%, 79.8\%, 84.0\%, and 86.2\% with models containing 10M, 19M, 83M, and 173M parameters, respectively. For instance, the 10M model outperforms the best existing SNN by 7.2\% on ImageNet, with training time acceleration and inference energy efficiency improved by 4.5$\times$ and 3.9$\times$, respectively. We validate the effectiveness and efficiency of the proposed method across various tasks, including object detection, semantic segmentation, and neuromorphic vision tasks. This work enables SNNs to match ANN performance while maintaining the low-power advantage, marking a significant step towards SNNs as a general visual backbone. Code is available at \href{https://github.com/BICLab/Spike-Driven-Transformer-V3}{Spike-driven Transformer V3.}
\end{abstract}

\begin{IEEEkeywords}
Spiking neural network, Spike-driven, Spiking Transformer, Neuromorphic computing, Efficient architecture and training
\end{IEEEkeywords}}

\maketitle
\IEEEdisplaynontitleabstractindextext
\IEEEpeerreviewmaketitle

\IEEEraisesectionheading{\section{Introduction}\label{sec:introduction}}
Spiking neural networks (SNNs) have garnered significant attention due to their brain-inspired spatio-temporal dynamics and spike-driven efficient computing paradigm \cite{Nature_2,schuman2022opportunities}. Spiking neurons communicate through binary spikes, following an "integrate-and-fire" dynamic model~\cite{Maass_1997_LIF}. Spatial and temporal information are integrated into membrane potentials by spiking neurons. When the membrane potential surpasses a certain threshold, the neuron emits a spike, and the membrane potential is reset according to specific rules. Neuromorphic chips \cite{2014TrueNorth,davies2018loihi,Nature_1} use spiking neurons and synapses as fundamental computing units to implement SNNs on hardware circuits. The spike-driven nature of neuromorphic computing enables SNNs to trigger sparse additions only when a signal is present, offering significant power efficiency advantages over traditional Artificial Neural Networks (ANNs). This low-power advantage is particularly evident in small-scale edge computing scenarios \cite{yin2021accurate,rao2022long,sbx2023}. For instance, the recently introduced asynchronous sensing-computing neuromorphic System on Chip (SoC) Speck \cite{Speck} has a resting power of just 0.42mW, and its operational power consumption in typical neuromorphic vision applications ranges from 0.7 to 15 mW.

In the era of deep learning, model scale is crucial to the success of modern machine intelligence. SNNs are also scaling up. There are two main training methods for large-scale SNNs: ANN2SNN and direct training. ANN2SNN approximates ANN activations using spike firing rates, achieving similar accuracy but requires numerous inference timesteps and abandoning spatio-temporal dynamics \cite{9597475,Rathi2020Enabling,hu2023fast,wu2021progressive,wang2023masked}. Direct training converts the nonlinear differential equations of spiking neurons into a discrete iterable form, and utilizing Spatio-Temporal BackPropagation (STBP) and surrogate gradients for training~\cite{wu2018spatio,10242251}. This method offers flexibility in architecture design and leverages spatio-temporal dynamics of spiking neurons \cite{fang2021deep,zheng2021going,qiu2023vtsnn,yao_attention_Pami,deng2024tensor,wu2024rsc,luo2024integer}, but multi-timestep simulations demand substantial computing resources. For example, direct training a spiking ResNet-19 with 10 timesteps requires approximately $20\times$ more memory than a traditional ResNet-19 \cite{fang2023spikingjelly}. 

On the path to scaling SNNs, current training methods face a paradox. Extended timesteps reduce errors from quantizing membrane potential into spikes, ensuring better task performance. However, multi-timestep simulations demand substantial GPU memory and are considerably more challenging to train, thereby impeding the scalability and performance improvement of SNNs.

Existing perspectives simply think that transforming continuous membrane potential into spikes results in quantization errors \cite{guo2022loss}. We extend this understanding by demonstrating that binary firing is a fundamental mechanism flaw and is the source of paradox. The binary nature of spike firing impairs the spiking neuron's ability to assess the significance of incoming signals (spatial representation), and constrains its memory and forgetting capabilities (temporal dynamics). Current methods, including ANN2SNN and direct training, attempt to address the spatial representation via rate coding \cite{deng2020rethinking} but remain inefficient training and unable to overcome temporal dynamics limitations.

We propose a simple Spike Firing Approximation (SFA) method to solve these flaws. Specifically, SFA uses integers as activations during training, which are then converted to spike trains during inference. This results in a fundamentally different spike firing pattern compared to ANN2SNN and direct training. Unlike the synchronous and random firing of the latter two, the SFA firing pattern is asynchronous and continuous. We demonstrate that this firing pattern significantly affects training, inference power, performance, network scaling, and neuromorphic chip implementation simultaneously. In all these aspects, the SFA approach outperforms ANN2SNN and direct training. 

We then focus on further enhancing the performance upper bound of SNNs through architecture design and model scale expansion. Meta-SpikeFormer \cite{meta_spikeformer} investigates the meta-design of SNN architecture based on spike-driven Transformer \cite{yao2023spike}, which for the first time integrates the spike-driven paradigm into Transformer. We upgrade the Meta-SpikeFormer, including improving the convolution and spike-driven self-attention operators, etc. The proposed Efficient Spiking Transformer (E-SpikeFormer) can simultaneously improve performance and energy efficiency. To solve the performance degradation problem of E-SpikeFormer when scaling up, we design a spike sparse convolution and combine it with a Masked Image Modeling (MIM) \cite{he2022masked,tian2022designing} pre-training strategy.

The proposed methods are validated on both static image classification (ImageNet-1k \cite{deng2009imagenet}), object detection (COCO \cite{lin2014microsoft}), semantic segmentation (ADE20K \cite{zhou2017scene}), and neuromorphic action recognition (HAR-DVS \cite{wang2022hardvs}) tasks. We scale up SNNs and achieve state-of-the-art results on all tested datasets. Our main contributions are:
\begin{itemize}
    \item[$\bullet$] 
    \textbf{Analysis of Spiking Neuron Mechanisms:} We identify binary spike firing as a mechanistic defect impacting both the spatial representation and the temporal dynamics of spiking neurons.

    \item[$\bullet$] 
    \textbf{Spike Firing Approximation (SFA):} \rone{The proposed SFA method leverages the equivalence between integer-based training and spike train inference. We demonstrate that SFA simultaneously optimizes training, power efficiency, performance, and scaling of SNNs by adjusting the spike firing patterns of spiking neurons. We also discuss the implementations of SFA on neuromorphic chips.}
    
    \item[$\bullet$] 
    \textbf{Efficient architecture design \rone{:}} We introduce the E-SpikeFormer, which achieves higher accuracy with lower energy cost based on the prior state-of-the-art SNN by designing efficient spike-driven convolution and self-attention operators.
    
    \item[$\bullet$] 
    \textbf{Scaling SNNs.} We integrate spike sparse convolution and MIM pre-training to address the challenge of performance degradation when scaling SNNs. 

    \item[$\bullet$] 
    \textbf{Enhancing the performance upper bound of SNNs:} We show that SNNs can achieve performance comparable to ANNs across various vision tasks with lower energy consumption, highlighting their potential as a general low-power visual backbone.
\end{itemize}

\begin{figure}[t]
  \centering
  \includegraphics[scale=0.42]{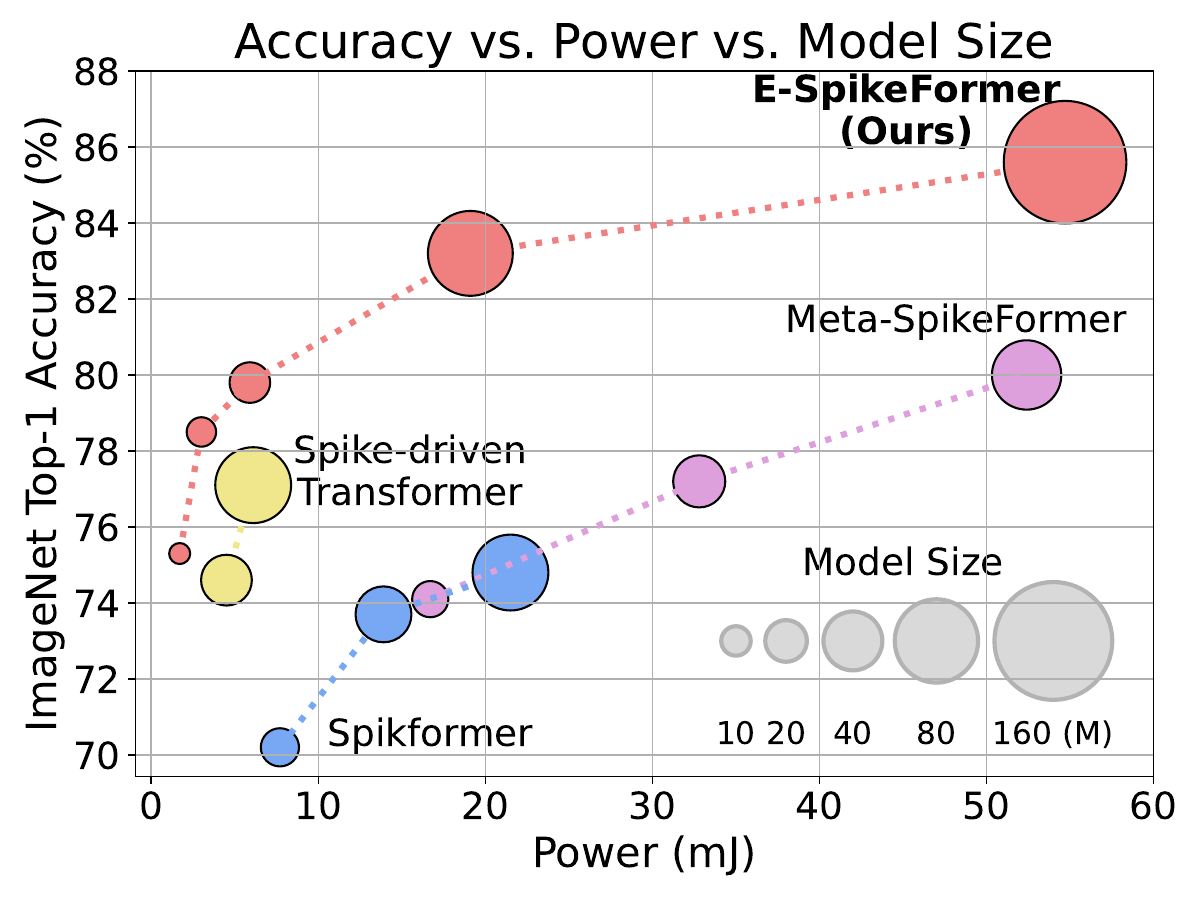}
  \caption{E-SpikeFormer versus other spiking Transformers on ImageNet-1k at $224^2$ input spatial resolution.}
  \label{fig:acc&latency&memor}
\end{figure}

The rest of the paper is organized as follows. Section~\ref{section:relate_work} reports related work. Section~\ref{section:methods} introduces our methods. Section~\ref{sec_SFA_analysis} analyzes the SFA method. Section~\ref{section_Experiment} verifies the effectiveness and efficiency of our methods on static and neuromorphic vision datasets. Section~\ref{sec_discussion} discusses the implementation of the proposed methods on neuromorphic chips. Section~\ref{section:conclusion} concludes this work.

\section{Preliminaries}\label{section:relate_work}
\textbf{Efficient Training in SNNs.} Complex spatio-temporal dynamics make the training cost of SNNs naturally larger than that of ANNs. The multi-timestep simulation requires a memory of $\mathcal{O}(L \times T)$ for STBP training, where $L$ is the layer and $T$ is the timestep. The underlying logic of efficient SNN training is to decouple training memory from timestep $T$. For example, based on the gradient importance analysis, only the spatial gradients of SNNs are kept \cite{xiao2022online,meng2023towards}; pre-training is performed on a single timestep and then the model is fine-tuned to multiple timesteps \cite{lin_2022_pretrain,meta_spikeformer}; temporal or spatial reversible design to drop memory usage \cite{zhang2024memory,hu2024high}; the timestep during training and inference is dynamically reduced based on sample difficulty \cite{chowdhury2022towards,li2024seenn}. These ideas are technically effective in alleviating the training challenges of large-scale SNNs but ignore a fundamental question: ``why do SNNs need to employ multi-timestep simulation in static tasks?" In this work, we provide an in-depth analysis of this problem and provide a solution. 

\textbf{Architecture Design in SNNs.} The main theme of SNNs in the deep learning era is scaling up. Architecture design is one of the important technical routes to reach this goal. Early efforts centered around the Spiking ResNet design. Directly copying the residual connection from ResNet \cite{he_resnet_2016} into SNNs \cite{zheng2021going} proved to be incapable of achieving identity mapping \cite{He2016}, and thus still suffered from performance degradation. To address this issue, SEW-ResNet \cite{fang2021deep} and MS-ResNet \cite{hu2024advancing} successfully extend the depth of SNNs to several hundred layers by creating shortcuts between spikes and membrane potentials between different layers, respectively. In contrast, recent efforts to combine SNN and Transformer \cite{vaswani_2017_attention_is_all_you_need,dosovitskiy2020image} have mainly focused on how to design spike-driven self-attention mechanisms. None of the earlier spiking Transformers \cite{zhang2022spike,wang2023complex,zhou2023spikformer} can be spike-driven. It was not until Spike-driven Transformer \cite{yao2023spike} implemented spike-driven self-attention using only addition and mask operations, thus making only sparse addition in spiking Transformer for the first time. Spike-driven Transformer v2 (i.e., Meta-SpikeFormer \cite{meta_spikeformer}) explores the meta-design\cite{yu2022metaformer,chou2024metala} of SNNs and discusses its potential impact on neuromorphic chip designs. In this work, we further optimize the redundancy design in the Meta-SpikeFormer to achieve higher performance with fewer parameters and power. 

\textbf{Masked Image Modeling (MIM) \cite{he2022masked}} is a self-supervised learning strategy that can effectively increase the size of visual models. Inspired by the masked pretrain-finetune language modeling in NLP \cite{devlin2018bert}, BEiT \cite{bao2021beit} first led the way in integrating MIM into the visual domain by training to predict visual tokens. Then, the famous MAE \cite{he2022masked} introduced an asymmetric encoder and decoder architecture, bypassing masked tokens in computation-intensive encoders, while all tokens are handled through a lightweight decoder. Pre-training with MAE-style facilitates the Vision Transformer (ViT) in bolstering its capacity for local bias and overcoming its performance degradation \cite{he2022masked,zhou2021image,wei2022masked}. MAE-style pre-training can also help CNN or CNN+ViT hybrid architectures scale up but requires adjustments to the training strategy. For example, MC-MAE \cite{gao2023mcmae}, Convnext v2 \cite{Woo_2023_CVPR}, and Spark \cite{tian2022designing} have employed masked or sparse convolutions to prevent information leakage and reconcile disparities between pre-training and fine-tuning. Since we follow the meta-architecture design of SNN in \cite{meta_spikeformer}, and the meta-SNN is in the hybrid form of Conv+ViT, we draw on the hybrid MIM framework in \cite{gao2023mcmae,Woo_2023_CVPR,tian2022designing} to realize large-scale pre-training of our E-SpikeFormer.

\section{Methods}\label{section:methods}
We introduce and analyze the spiking neuron models in Section~ \ref{subsection:spiking_neuron}. Then we propose the new training method, E-SpikeFormer architecture, and MIM strategy for scaling SNNs in Section~\ref{subsection_SFA}, \ref{subsection:E_spikeformer}, and \ref{subsection_scale_training}, respectively.

\subsection{Spiking Neuron}\label{subsection:spiking_neuron}
\subsubsection{Integration, Firing and Reset Mechanisms}
The spatio-temporal dynamics and spike communication of spiking neurons are abstracted from biological neurons \cite{Maass_1997_LIF}. Due to the high computational complexity of biological neuron dynamics, existing spiking neurons are simplified into differential equations to facilitate computer simulations. By solving the differential equation, an iterative version of the spiking neuron can be obtained \cite{wu2018spatio,neftci2019surrogate}:
\begin{align}
    \mathbf{U}[t] &= \beta \mathbf{H}[t-1] + \mathbf{X}[t], \label{eq:U_definition} \\
    \mathbf{S}[t] &= \text{Hea}(\mathbf{U}[t]-V_{\text{th}}), \label{eq:S_definition} \\
    \mathbf{H}[t] &= 
        \begin{cases}
            \mathbf{U}[t]\cdot(1-\mathbf{S}[t]) + V_{\text{reset}}\cdot\mathbf{S}[t], & \text{hard reset}, \label{eq:H_hard_reset} \\
            \mathbf{U}[t]-V_{\text{th}}\cdot\mathbf{S}[t], & \text{soft reset},
        \end{cases} 
\end{align}
where $\textit{t}$ is the timestep, $\mathbf{U}[t]$ represents the membrane potential after charging but before firing. Spatial input $\mathbf{X}[t]$ is extracted from the original spike input through a Conv or MLP operation, temporal input $\beta \mathbf{H}[t-1]$ is derived from the decay of the membrane potential at the previous timestep, $\beta$ is the decay (leaky) factor, $\mathbf{S}[t]$ is the output spike, $\mathbf{H}[t]$ is the membrane potential after firing, $\text{Hea}(\cdot)$ is the Heaviside function that $\text{Hea}(\textit{x})=1$ if $x\geq0$. When $\mathbf{U}[t]$ exceeds the threshold $V_{\text{th}}$, spiking neurons will fire spikes, and the membrane potential is reset to $V_{\text{reset}}$. 

Eq.~\eqref{eq:U_definition}, Eq.~\eqref{eq:S_definition}, Eq.~\eqref{eq:H_hard_reset} represent the integration, firing and reset mechanisms of spiking neurons, respectively:

\romannumeral1) Integration defines how spatio-temporal information is fused. For example, when using hard reset, in Integrate-and-Fire (IF) neurons \cite{gerstner2014neuronal}, $\beta = 1$; whereas in Leaky-Integrate-and-Fire (LIF) \cite{Maass_1997_LIF}, $\beta \not= 1$. Thus, LIF and IF memorize information from previous timesteps differently. 

\romannumeral2) Firing defines how information is transmitted in SNNs. Spiking neurons fire binary spikes only when the membrane potential exceeds the threshold. SNNs then can enjoy the low-power nature brought by event (spike)-driven, which means that the network only triggers sparse addition when there is input~\cite{indiveri2011neuromorphic,yao2023spike}.

\romannumeral3) Reset defines how a spiking neuron memorizes and forgets (Fig.~\ref{Fig_reset}). We discuss three reset mechanisms: no reset \cite{fang2023parallel}, hard reset \cite{Maass_1997_LIF}, soft reset \cite{gerstner2014neuronal}. No reset and hard reset represent two extreme cases. The former remembers all information, resulting in infinite firing over time. The latter forgets any existing information after a neuron fires a spike (if $V_{\text{reset}} = 0$). As a trade-off between no reset and hard reset, soft reset preserves some information while subtracting $V_{\text{th}}$ from the membrane potential after firing a spike.

\subsubsection{Inherent Flaws in Binary Spike Firing}\label{subsec:inherent_flaw}
Spiking neuron models are typically defined by their integration and reset mechanisms. Based on Eq.~\eqref{eq:U_definition} and Eq.~\eqref{eq:H_hard_reset}, four types of spiking neurons can be identified: LIF with Hard Reset (LIF-HR); LIF with Soft Reset (LIF-SR); IF with Hard Reset (IF-HR); IF with Soft Reset (IF-SR). 

\begin{figure*}[t]
\centering
\subfigure[Spatial input.]{\includegraphics[scale=0.7]{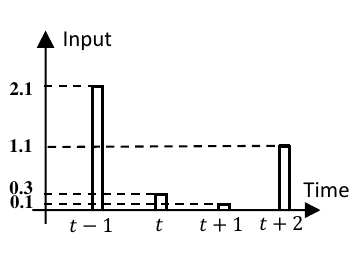}}
\subfigure[No reset.]{\includegraphics[scale=0.7]{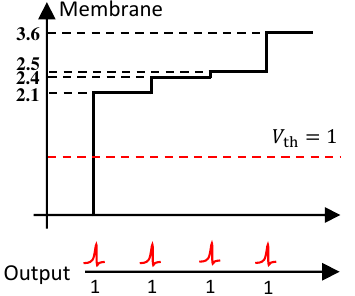}}
\subfigure[Hard reset.]{\includegraphics[scale=0.7]{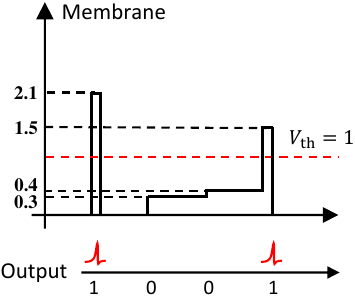}}
\subfigure[Soft reset.]{\includegraphics[scale=0.7]{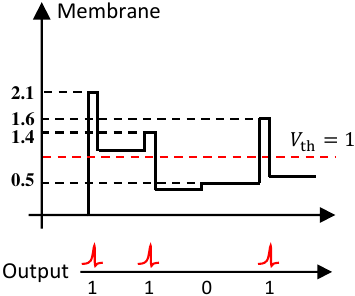}}
\caption{The impact of the reset mechanism on the spatio-temporal dynamics of spiking neurons. (a) The $(t-1)$-th and $(t+2)$-th inputs are greater than the threshold, the $t$-th and $(t+1)$-th inputs are not. Here we set $\beta = 1$. (b) If no reset, the spiking neuron will keep firing after the input exceeds the threshold at some point. (c) If hard reset to zero, a spike will be fired whenever the spatial input at the current timestep exceeds the threshold. Inputs that do not exceed the threshold will not trigger a spike. (d) If there is a soft reset, the membrane potential at $t$-th timestep contains part of the information of $(t-1)$-th timestep; even if the input at $t$-th is small, a spike may be fired. 
}
\label{Fig_reset}
\end{figure*}

An often overlooked fact is that the firing mechanism (Eq.~\eqref{eq:S_definition}) also influences the integration and reset mechanisms. Consequently, binary spike firing introduces at least two intrinsic flaws for spiking neurons:

\romannumeral1) \textbf{Spatial Representation.} It is impossible to assess the importance of the current spatial input. Regardless of how large the membrane potential is, the output is just a spike. This can also be interpreted as a large quantization error in the conversion of the membrane potential to a spike.

\romannumeral2) \textbf{Temporal Dynamics.} It is impossible to perform data-dependent memorization or forgetting. As shown in Eq.~\eqref{eq:H_hard_reset}, $\mathbf{S}[t]$ can only take on the value 0 or 1, implying that the soft reset can only forget fixed values.

Let’s take the spiking neurons in Fig.~\ref{Fig_reset} as an example. Assume that the spatial input at $(t-1)$-th is very important (large value). In the case of the hard reset, the spiking neuron outputs a spike and clears the membrane potential to $V_{reset}$. Obviously, this has a large quantization error and cannot reflect the fact that the input is very important. In the case of the soft reset, the spiking neuron outputs a spike, but the membrane potential passed to the next timestep is still large. One possibility is that the spiking neuron fires for a long time afterward, and then you cannot tell whether the firing at $t$-th is caused by the $(t-1)$-th spatial input or by the current spatial input (as shown in the example in Fig.~\ref{Fig_reset}(d)). Thus, the soft reset may result in identical memory and forgetting processing at each timestep.

\subsection{Spike Firing Approximation (SFA) Method}\label{subsection_SFA}

Previous SNNs exploited multi-bit spikes (integers) \cite{ponghiran2022spiking,shen2023conventional} or continuous value \cite{wu2021liaf} for the motivation of reducing quantization error. This in fact mitigates the flaws of binary spikes. But, the concern that comes with it is that the spike-driven nature is destroyed. We address these concerns via three steps: Step-1, we assume that the signal transmitted between spiking neurons is the spike firing rate. Step-2, during training, we approximate and replace the above spike firing rate with integer-valued activations as the forward-propagation signal. Step-3, during inference, we preserve the spike-driven nature by equating the integer activation to binary spikes over multi-timesteps. We name the whole process ``Spike Firing Approximation (SFA)".

\begin{figure}[t]
    \centering
    \includegraphics[scale=0.85]{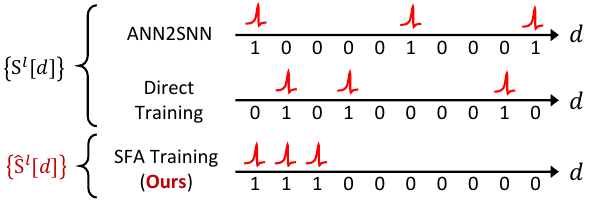}
    \caption{Spike firing patterns. Assuming an approximation of 0.3 ($a^l_D=0.3$), i.e., three spikes need to be fired at ten timesteps ($D=10$). The timesteps at which the three spikes appear are synchronized and randomized for ANN2SNN and direct training SNN, where synchronization means that all ten timesteps must be completed to compute the spike firing rate. In contrast, SFA training will not fire spikes on the remaining seven timesteps after firing them on the first three timesteps. Therefore, the spike firing in SFA can be realized in an asynchronous manner (detailed in Section~\ref{sec_hardware_analysis}).}
    \label{Fig_spike_firing_pattern}
\end{figure}

\textbf{SFA Training and Inference.} Existing ANN2SNNs\cite{wu2021progressive} and vanilla direct training SNNs\cite{wu2019direct,rathi2021diet,wei2024event} can be understood as rate coding (Fig.~\ref{Fig_spike_firing_pattern}). Inspired by this, we assume that spiking neurons transmit spike firing rates (continuous values) in the spatial dimension.

In Step-1, The spike firing rate $a^l_D$ of a spiking neuron is defined as the ratio of fired spikes to the total timestep:
\begin{equation}
    a^l_D=\frac{1}{D}\sum^D_{d=1}\mathbf{S}^l[d],
    \label{eq:spike_firing_rate}
\end{equation}
where $\{\mathbf{S}^l[d]\}_D$ is the spike train in ANN2SNN or direct training over given $D$ timesteps at $l$-th layer. 

In Step-2, we perform integer activation training, i.e., employ an integer to approximate $a^l_D$. We replace $\sum^D_{d=1}\mathbf{S}^l[d]$ with the integer-valued activation of a single timestep:
\begin{equation}
    \rone{\mathbf{S}^l_D = \text{Fire}_D(\mathbf{U}^l) = \lfloor \text{clip}\{\mathbf{U}^l,0,D\}\rceil,}
    \label{eq:integer_replace_spike_train}
\end{equation}
where $\mathbf{S}^l_D$ is the integer activation within $[0,D]$, $\mathbf{U}^l$ is the membrane potential before firing, $\text{Fire}_D(\cdot)$ is the integer-valued fire function, $D$ indicates the maximum integer value allowed to be fired. And $\text{clip}\{\mathbf{U}^l,0,D\}$ confines the value of $\mathbf{U}^l$ within 0 and $D$, $\lfloor\cdot\rceil$ rounds to the nearest integer. 

In Step-3, we perform spike activation inference:
\begin{equation}
    \mathbf{S}^l_D = \sum^D_{d=1}\hat{\mathbf{S}}^l[d],
    \label{eq:integer_to_spike_train}
\end{equation}
where $\{\hat{\mathbf{S}}^l[d]\}_D$ is the spike train in our SFA training. The spatial input to the spiking neuron at ($l+1$)-th layer is
\begin{equation}
    \mathbf{X}^{l+1} = \mathbf{W}^{l+1}a^l_D = \mathbf{W}^{l+1}\frac{1}{D}\mathbf{S}^l_D = (\frac{1}{D}\mathbf{W}^{l+1})\sum_{d=1}^D\hat{\mathbf{S}}^l[d]. \label{eq:input_of_neuron}
\end{equation}

As a simple summary, the SFA method replaces $\mathbf{S}[t]$ in Eq.~\eqref{eq:S_definition} with $\mathbf{S}^l_D$ during training; when inference, each timestep is expanded into $1 \times D$ timesteps, and the corresponding spike train is $\{\hat{\mathbf{S}}^l[d]\}_D$. \rone{The integer firing of SFA not only reduces the quantization error when converting the membrane potential of spiking neurons into spikes, but also enhances the forgetting capability of spiking neurons at the current timestep to address temporal dynamics deficiencies.} More details of SFA method are given in Section~\ref{sec_SFA_analysis}. In addition, in Section~\ref{section_Experiment}, we systematically analyze that replacing the existing $\{\mathbf{S}^l[d]\}_D$ with our $\{\hat{\mathbf{S}}^l[d]\}_D$ spike firing pattern (Fig.~\ref{Fig_spike_firing_pattern}) has a profound impact on the training, performance, energy consumption, and scaling of SNNs.

\begin{figure*}[t]
\centering
\includegraphics[scale=0.75]{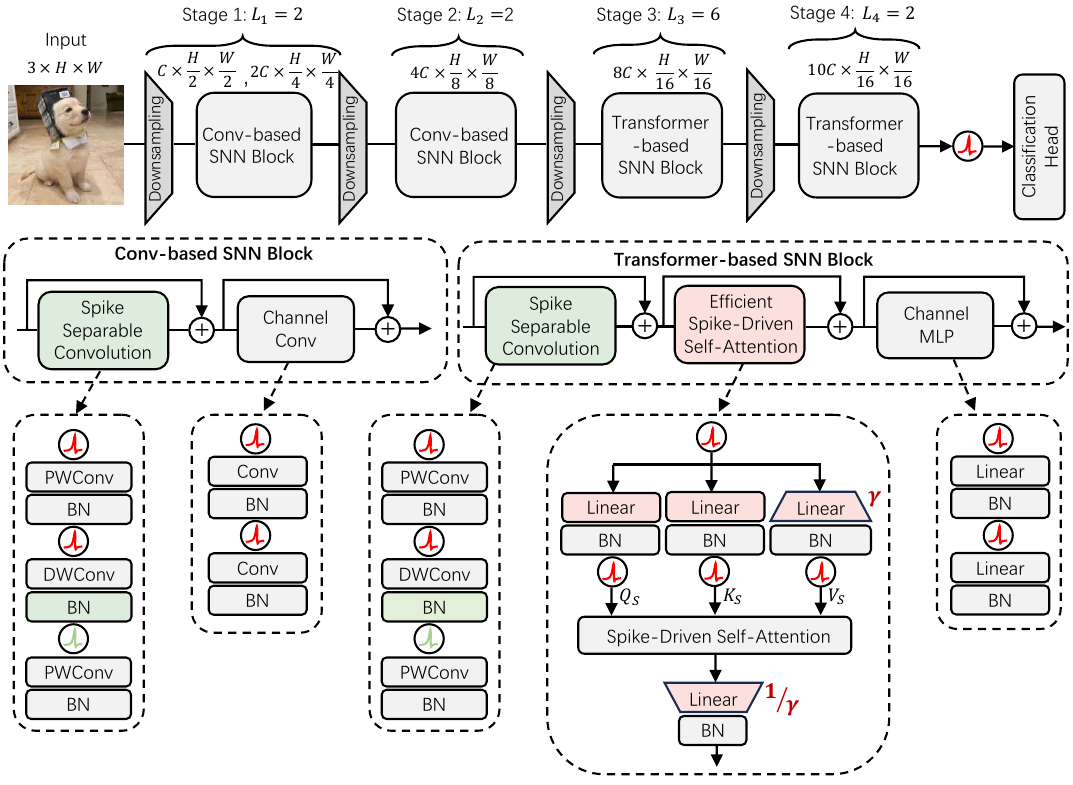}
\caption{The overview of E-SpikeFormer. In general, we follow the design of Meta-SpikeFormer \cite{meta_spikeformer} on a macro level, such as CNN-based and Transformer-based SNN blocks and their proportions at each stage. We redesign the interior of CNN-based and Transformer SNN blocks with the goal of efficient architecture. In this figure, we have marked all upgraded designs in \emph{green} or \emph{pink}. In the green part, we insert a BN layer and a spiking neuron layer after the DWConv layer in the CNN-based block to avoid the increase in power caused by the larger convolution kernel size after re-parameterization. In the pink part, the Re-parameterization Convolution (RepConv) in \cite{meta_spikeformer} is replaced by a linear layer to reduce power, and we compensate for the potential performance loss by expanding the channel number of $V_{S}$. In addition, we add a Spike Separable Convolution (SpikeSepConv) module to the Transformer-based SNN block in \cite{meta_spikeformer}, which helps performance.}
\label{fig_spike_driven_v3_arch}
\end{figure*}

\subsection{E-SpikeFormer}\label{subsection:E_spikeformer}
Meta-SpikeFormer \cite{meta_spikeformer} explores the meta-design of SNNs in Spike-Driven Self-Attention (SDSA), architecture, and shortcut. We here focus on designing a more efficient SNN based on Meta-SpikeFormer, and name it "E-SpikeFormer". 

\textbf{Overview.} E-SpikeFormer contains Conv-based and Transformer-based SNN stages (Fig.~\ref{fig_spike_driven_v3_arch}). Each stage contains several blocks, each block containing a token mixer module and a channel mixer module. Many Reparameterization Convolutions (RepConv) are used in Meta-SpikeFormer, which results in a significant increase in power. According to Table~5 in Meta-SpikeFormer, the RepConv part accounts for approximately 60\% of the total energy cost. In this work, we eliminate all RepConvs and compensate for performance by incorporating a set of additional lightweight Spike Separable Convolution (SpikeSepConv). In Fig.~\ref{fig_spike_driven_v3_arch}, we label all upgraded designs of our E-SpikeFormer relative to Meta-SpikeFormer as \emph{green} or \emph{pink}. We see that there are three primary modifications in our approach: \romannumeral1) In the Conv-based SNN block, the token mixer has been replaced with a SpikeSepConv; \romannumeral2) An SpikeSepConv module has been added prior to the token mixer (i.e., SDSA module) in the Transformer-based SNN block; \romannumeral3) In the SDSA module, RepConv has been substituted with a linear operation.

\textbf{Conv-based SNN Block} in \cite{meta_spikeformer} uses separable convolution \cite{sandler2018mobilenetv2} as the token mixer. However, in SNNs, the direct connection between depthwise and pointwise convolutions with subsequent reparameterization poses a serious energy consumption problem. Our solution is to add a Batch Normalization (BN) layer \cite{ioffe_batchNorm_2015} and a spiking neuron layer between depthwise and pointwise convolutions. Specifically, our CNN-based SNN block is written as:
\begin{align}
&\mathbf{U}' = \mathbf{U} + \text{SpikeSepConv}(\mathbf{U}),\\
&\mathbf{U}'' = \mathbf{U}' + \text{ChannelConv}(\mathbf{U}'),\\
&\text{SpikeSepConv}(\mathbf{U}) = 
\begin{aligned}[t]
    &\text{Conv}_\text{pw2} \big( \text{SN} \big( \text{Conv}_\text{dw} \big(  \\
    &\text{SN} \big(\text{Conv}_\text{pw1} \big( \text{SN} (\mathbf{U}) \big) \big)) \big) \big),
\end{aligned}\\
&\text{ChannelConv}(\mathbf{U}') = \text{Conv}(\text{SN}(\text{Conv}(\text{SN}(\mathbf{U}')))),
\end{align}
where $\text{SpikeSepConv}(\cdot)$ is the token mixer, $\text{ChannelConv}(\cdot)$ is the channel mixer, $\text{Conv}_\text{pw1}(\cdot)$ and $\text{Conv}_\text{pw2}(\cdot)$ are pointwise convolutions, $\text{Conv}_\text{dw}(\cdot)$ is depthwise convolution \cite{chollet2017xception}, $\text{Conv}(\cdot)$ is the vanilla convolution. ${\text{SN}}(\cdot)$ is the spike neuron layer. Note, we ignore the BN layer for ease of writing.

\begin{figure*}[t]
\centering
\includegraphics[width=0.78\textwidth]{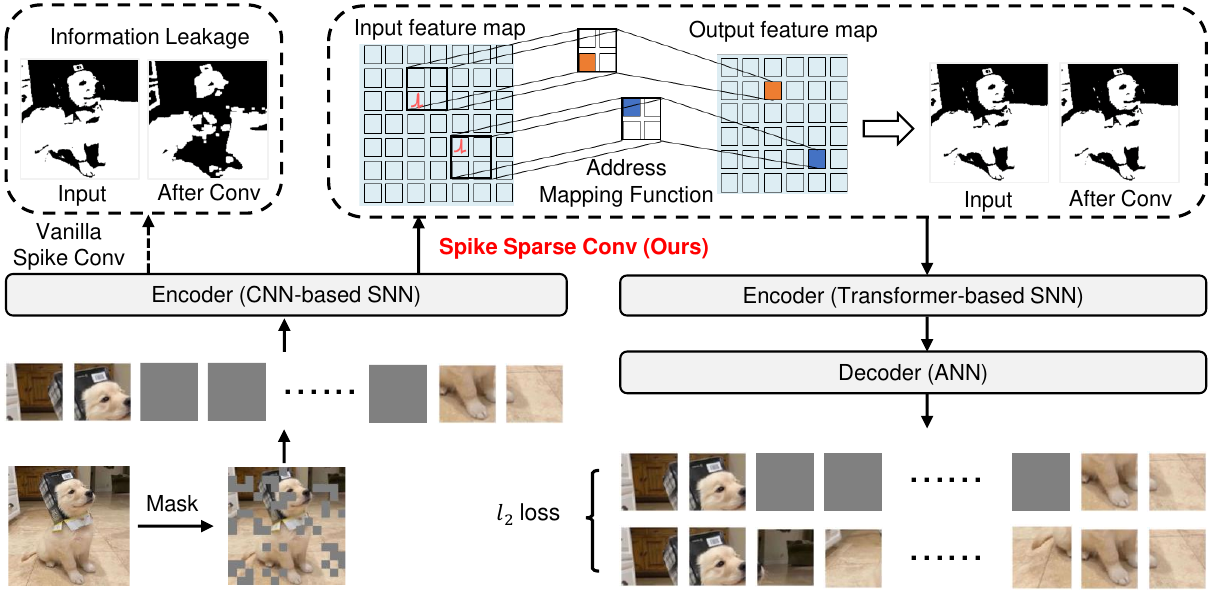}
\caption{The overview of MIM pre-train in E-SpikeFormer. It consists of a SNN encoder and a ANN decoder. The encoder processes only the visible pixels. The decoder reconstructs the images to learn representations, which will be removed in the fine-tuning stage. Using Vanilla Spike Convolution (VSC) as an encoder will lead to information leakage, so we designed Spike Sparse Convolution (SSC) to avoid this flaw. White and black in the figure represent the one/zero area respectively. Exploiting VSC will cause the black zero area to continue to grow as the depth increases until there is no valid information. We show that SSC is a natural fit for neuromorphic chips due to the spike-driven nature of SNNs (Section~\ref{subsec_scale_2}).
}
\label{fig_MIM_SNN}
\end{figure*}

\textbf{Transformer-based SNN Block.} Motivated by the observation that adding a convolution layer prior to the self-attention operator to extract local features enhances the integration of global features \cite{cai2022efficientvit,wadekar2022mobilevitv3}, we introduce a lightweight SSC module before the SDSA. This addition aims to offset the performance decline resulting from the elimination of RepConv. Then, the Transformer-based SNN Block can be written as
\begin{equation*}
\begin{aligned}
    \mathbf{U}^{\prime} &= \mathbf{U} + \text{SpikeSepConv}(\mathbf{U}), \\
    \mathbf{U}^{\prime\prime} &= \mathbf{U}^{\prime} + \text{E-SDSA}(\mathbf{U}^{\prime}), \\
    \mathbf{U}^{\prime\prime\prime} &= \mathbf{U}^{\prime\prime} + \text{ChannelMLP}(\mathbf{U}^{\prime\prime}),
\end{aligned}
\end{equation*}
where $\text{E-SDSA}(\cdot)$ is the proposed E-SDSA module. 

\textbf{Efficient Spike-Driven Self-Attention (E-SDSA) Module.} In this work, we combine the generation of $\mathbf{K_S}, \mathbf{Q_S}, \mathbf{V_S}$, the SDSA operator, and the subsequent linear operator as the E-SDSA module (pink module in Fig.~\ref{fig_spike_driven_v3_arch}). In E-SDSA, we utilize linear operations to generate $\mathbf{K_S}, \mathbf{Q_S}, \mathbf{V_S}$ instead of energy-consuming reparameterization convolutions. We adopt the iRMB technique in \cite{zhang2023rethinking} that expands the channels of $\mathbf{V_S}$ by a factor of $\gamma$ to enhance the representation of E-SDSA. Specifically, the E-SDSA module can be written as
\begin{equation*}
\begin{aligned}
    & \mathbf{Q_S} = \text{SN}(\text{Linear}(\mathbf{U})), \mathbf{K_S} = \text{SN}(\text{Linear}(\mathbf{U})), \\
    & \mathbf{V_S} = \text{SN}(\text{Linear}_\gamma(\mathbf{U})), \\
    & \mathbf{U}^{\prime} = \text{Linear}_{\frac{1}{\gamma}}(\text{SN}(\mathbf{Q_S}\mathbf{K_S}^T\mathbf{V_S}*scale)),
\end{aligned}
\end{equation*}
where Linear($\cdot$) is the linear operator with learnable matrix $\mathbf{W} \in \mathbb{R}^{C_{in} \times C_{out}}$ and bias $\mathbf{b} \in \mathbb{R}^{C_{out}}$ ($C_{in}$ and $C_{out}$ are the dimension of input and output respectively. Its mathematics form can be written as $\mathbf{y}=\mathbf{W}\mathbf{x}+\mathbf{b}$. Similarly, $\text{Linear}_\gamma(\cdot)$ has weight of $\mathbf{W}_\gamma \in \mathbb{R}^{C_{in} \times \gamma C_{out}}$ and bias of $\mathbf{b}_\gamma \in \mathbb{R}^{\gamma C_{out}}$. $\text{Linear}_\frac{1}{\gamma}(\cdot)$ has weight of $\mathbf{W}_{\frac{1}{\gamma}} \in \mathbb{R}^{\gamma C_{out} \times C_{in}}$ and bias of $\mathbf{b}_{\frac{1}{\gamma}} \in \mathbb{R}^{C_{in}}$. $\text{SN}(\mathbf{Q_S}\mathbf{K_S}^T\mathbf{V_S}*scale)$ is the SDSA operator, which is consistent with Meta-SpikeFormer \cite{meta_spikeformer}, and constant \emph{scale} factor can be reparameterized into SN's threshold.

\subsection{Spike Sparse Convolution for Scaling Up SNNs}\label{subsection_scale_training}
Directly scaling the ViT leads to performance degradation \cite{dosovitskiy2020image,touvron2021training}. The proposed E-SpikeFormer was not immune to this problem. The pretrain-finetune MIM paradigm is one of the common methods to solve the degradation problem \cite{he2022masked}. Tian \etal \cite{tian2022designing} pointed out that the MIM strategy cannot be used directly in CNNs. An input image is divided into non-overlapping patches in ViTs. Benefiting from ViT's ability to handle irregular and non-overlapping patches, MIM can simply drop masked patches to remove the information. In contrast, CNNs not only operate on regular grids but also perform sliding windows with overlapping. Directly employing MIM in CNNs would result in \emph{information leakage}. That is, since the input contains lots of blank (masked) areas, deep convolutions can hardly learn any features as the network deepens. Tian \etal proposed a solution SparK \cite{tian2022designing} that involves encoding the input image, after applying a zero-mask, using the sparse convolution encoder.

The proposed MIM-SNN follows SparK (Fig.~\ref{fig_MIM_SNN}) but with a newly designed Spike Sparse Convolution (SSC) encoder. The overall process can be written as ${G}=\mathcal{F}_{E}\circ \mathcal{F}_{D}$, which embodies an encoder-decoder architecture. An encoder $\mathcal{F}_{E}$ is responsible for mapping masked inputs $\mathbf{X}_1$ to a latent feature $\mathbf{Z}=\mathcal{F}_{E}(\mathbf{X}_1)$. A decoder $\mathcal{F}_{D}$ reverses this process, mapping $\mathbf{Z}$ back to the pixel space to reconstruct the target view $\mathbf{X}_2$ ($\mathbf{X}_1$ and $\mathbf{X}_2$ are complementary). 

\textbf{Masking.} Patch-wise masking is the first step of MIM. An image $\mathbf{X}$ is restructured into a succession of flattened 2D patches, $ \mathbf{X} \in \mathbb{R}^{r \times p}$, where $p$ denotes the patch size. Then a random binary mask $m \in \{0,1\}^n$ is generated, enabling the acquisition of two complementary \cite{zhang2022mask} masked views of $\mathbf{X}$ :
\begin{equation}
     \mathbf{X}_1=\mathbf{X}[n]\in \mathbb{R}^{r_1 \times p},\quad 
 \mathbf{X}_2=\mathbf{X}[1-n]\in \mathbb{R}^{r_2 \times p},
 \label{eq:mask_view}
\end{equation}
$r_1$ and $r_2$ are integers that $r_1=n(1-\mu)$ and $r_2=n\mu$, ensuring that their sum equals $n$, $\mu$ is the masking ratio.

\textbf{Encoder.} The masked image is similar to neuromorphic data, which only has events when the brightness changes and therefore contains large blank (zero) areas. Intuitively, convolutions in SNNs are naturally sparse convolutions due to their spike-driven nature. As shown in Fig.~\ref{fig_MIM_SNN}, according to the address of the input spike, the address mapping function \cite{Speck} outputs the address of the neuron and synapse that need to perform sparse synaptic addition operations. Unfortunately, using Vanilla Spike Convolution (VSC) cannot avoid information leakage in MIM-SNN. This discrepancy arises from the fact that the VSC would perform convolution at all positions, eventually causing information leakage (it can also be called spike degradation \cite{hu2024advancing}, the spiking neurons in deep layers almost no longer fire spikes). 

To address this challenge, we propose the SSC, which only performs convolution on positions that are not masked. We reshape $\mathbf{X}_1$ with $C_{\mathrm{in}}$ channels with height $H$ and width $W$ in the $d$ dimensional spatial space. We process this feature by SSC with kernel weights $\mathbf{W}\in \mathbb{R}^{\mathnormal{K}^d\times C_{\mathrm{in}} \times C_{\mathrm{out}}}$. As an example, in the 2D coordinate
space, $\bf W$ contains $ C_{\mathrm{in}} \times C_{\mathrm{out}}$ channels with $K$ kernel size and $\mathnormal{K}^d$=$3^2$.
The convolution process of SSC is represented as
\begin{equation}
    \mathbf{Y}_{p} = \sum_{k\in K_{d}}\alpha (\mathbf{W}_{k} \cdot \mathbf{X}_1^{p_{k}}),
\end{equation}
where $k$ enumerates all discrete locations in the kernel space $\mathnormal{K}^d$. $p_k=p+k$ is the corresponding location around center $p$, where $k$ is an offset distance from $p$. $\alpha \in \{0,1\}$ is a selector. When the center $ p \in \mathbf{X}_1 $ is the masked area, $\alpha$ is 0, indicating that the point $p$ does not participate in the computation, preserving the masked feature map. $\alpha=1$ denotes that the region is unmasked and active for computation. During pre-training, the encoder exploits SSC; during fine-tuning, the weights are converted back to VSC. 

\textbf{Decoder.} The output of the encoder is supplied to the decoder along with learnable masked tokens. Then the masked tokens are reconstructed into an image. There are usually two options for the decoder: plain \cite{gao2023mcmae} and multi-scale \cite{tian2022designing} ANN. We found that different decoders have less impact on the accuracy of MIM-SNN pre-train. So during pre-training, we employed a plain ANN (Transformer) as a decoder, which is removed in the fine-tuning and would not impact SNNs' spike-driven inference.

\textbf{Loss Function.} In pre-training, we use MSE loss as the metric function $\mathcal{L}$ and normalized pixel values of each masked patch as the prediction target.
\begin{equation}
    \mathcal{L}_{G}=\mathbb{E}_{ \mathbf{X}}\mathbb{E}_{\mathbf{X}_1,\mathbf{X}_2|\mathbf{X}}\left\| \mathcal{F}_{D}(\mathcal{F}_{E}(\mathbf{X}_1))-\mathbf{X}_2\right\|^2,
    \label{eq:mim-loss}
\end{equation}
\rone{where $\mathcal{F}_{D}(\mathcal{F}_{E}(\mathbf{X}_1))$ is the decoder output and $\| \cdot\|^2$ is $l_2$ norm}.  
\par
Then we exploit the distillation \cite{touvron2021training} technique during fine-tuning, which is the same in Meta-SpikeFormer \cite{meta_spikeformer}.

\textbf{Analysis on MIM-SNN.} In Section~\ref{subsec_SFA_scale}, we reveal how binary spiking affects the scaling of SNNs. In Section~\ref{subsec_scale_2}, we show that the difference between SSC and our VSC lies in the address mapping function \cite{Speck}, thus neuromorphic chips are naturally capable of performing SSC-based MIM-SNN pre-train.

\section{Analysis of Spike Firing Approximation}\label{sec_SFA_analysis}

First, we demonstrate the equivalence between integer-based training and the corresponding spike train inference within the SFA method. Subsequently, we analyze the approximation error that arises from substituting the spike firing rate with integer activations.

\subsection{Implementation of Spike-driven Inference}\label{subsec_dynamic_soft_reset}

The transition between integer training and spike inference can be losslessly achieved by IF-SR neurons with $V_{th} = 1$. 

\begin{proposition} \label{propo_step2_to_step3}
Consider the $\text{Fire}_D(\cdot)$ function at $l$-th layer in SNN, its integer-value output is equal to the sum of spikes generated by IF-SR spiking neuron with $D$ timesteps:
\begin{equation}
    \mathbf{S}^l_D = \sum_{d=1}^D\hat{\mathbf{S}}^l[d],
\end{equation}
where $\mathbf{S}^l_D$ is the integer value fired by $\text{Fire}_D(\cdot)$ at $l$-th layer and $\{\hat{\mathbf{S}}^l[d]\}_D$ is the spike train generated by IF-SR spiking neuron over given $D$ timesteps with $V_{th}=1$.
\end{proposition}

\begin{proof}\label{prop2_proof}
Feeding the membrane potential input of Eq.~\eqref{eq:integer_replace_spike_train} into an IF-SR neuron with $V_{th}=1$, the spike train generated is:
\begin{equation}
    \{\hat{\mathbf{S}}^l[d]\}_D = \text{IF-SR}(\mathbf{U}^l),
\end{equation}
which means that the IF-SR spiking neuron subtracts 1 at each timestep until the membrane potential is less than the threshold. The underlying assumption is that there is non-zero input only at timestep $d=1$, and the input is zero at timesteps $d=2, \cdots, D$. Then the sum of the spike train is:
\begin{equation}
    \sum_{d=1}^D\hat{\mathbf{S}}^l[d] = \lfloor\text{clip}\{\mathbf{U}^l,0,D\}\rfloor,
    \label{eq:IF_SR_spike_sum}
\end{equation}
where $\lfloor\cdot\rfloor$ is the floor function. 
According to Eq. \eqref{eq:integer_replace_spike_train}, we have:
\begin{align}
    \mathbf{S}^l_D &= \lfloor\text{clip}\{\mathbf{U}^l,0,D\}\rceil, \label{eq:eq5} \\
    &= \lfloor\text{clip}\{\mathbf{U}^l+0.5,0,D\}\rfloor, \label{eq:round_to_floor} \\
    &= \lfloor\text{clip}\{\mathbf{U}^l,0,D\}\rfloor, \label{eq:bias_in_weight}
\end{align}
where in Eq. \eqref{eq:bias_in_weight}, the bias '0.5' about the equivalence between the round function $\lfloor\cdot\rceil$ and floor function $\lfloor\cdot\rfloor$ can be incorporated into the weight. Combining Eq. \eqref{eq:IF_SR_spike_sum} and Eq. \eqref{eq:bias_in_weight}, we have:
\begin{equation}
    \mathbf{S}^l_D= \lfloor\text{clip}\{\mathbf{U}^l,0,D\}\rfloor = \sum_{d=1}^D\hat{\mathbf{S}}^l[d]. \label{eq_integer_to_spike}
\end{equation}
\end{proof}

Eq.~\eqref{eq_integer_to_spike} implies the equivalence from integer training to spike train inference. We analyze how this affects training, depending on the task type.

\textbf{One-timestep SFA training for static vision tasks.} Static vision tasks, such as classification and detection, do not contain temporal information. As shown in Fig.~\ref{fig_vanilla_training}, to overcome the defect of spiking neurons in spatial representation, previous direct training forcibly constructed multi-timestep rate encoding through repeated input. Since SFA training changes the spike firing pattern, for static tasks we only need one-timestep training and then convert it to multi-timestep spike-driven inference (Fig.~\ref{fig_sfa_training}). SFA training has two obvious advantages over vanilla training: First, the one-timestep training cost is lower; Second, during inference, there is no need to repeatedly input images in SFA.

\begin{figure*}[t]
\centering
\subfigure[Vanilla training ($D=1$).]{\label{fig_vanilla_training}\includegraphics[scale=0.85]{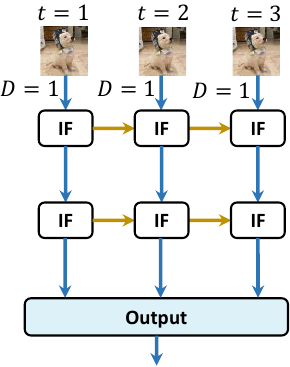}}
\quad \quad\quad\quad\quad \quad
\subfigure[SFA training on static vision tasks ($T \times D = 1\times3$).]{\label{fig_sfa_training}\includegraphics[scale=0.85]{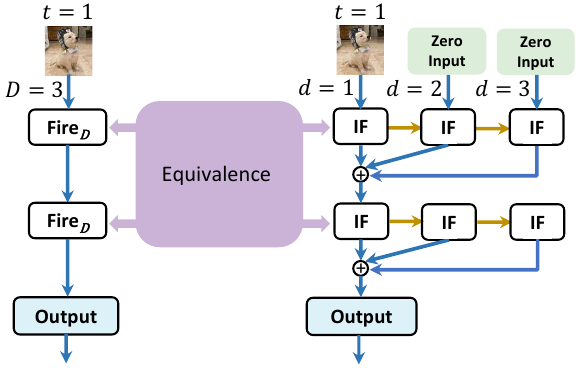}}\\
\quad
\subfigure[SFA direct training on event-based vision ($T \times D = 2\times3$).]
{\label{fig_sfa_training_multi}\includegraphics[scale=0.85]{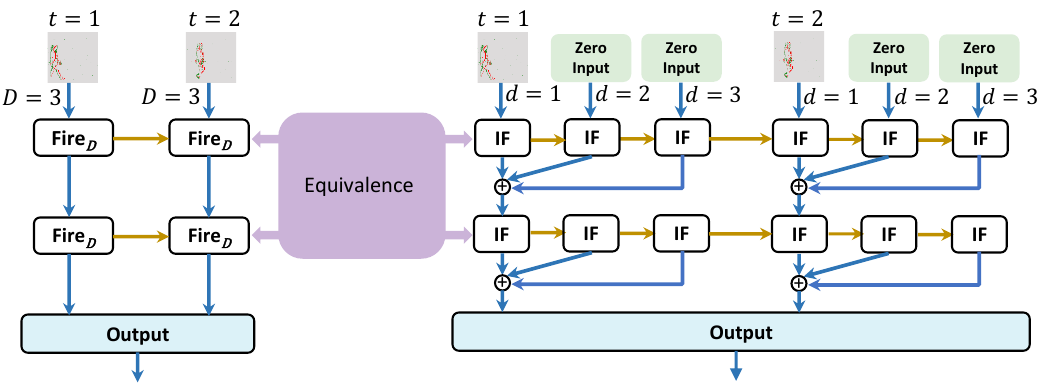}}
\caption{Comparison of vanilla and SFA training in SNNs. $T$ in SNNs typically denotes the timestep. For vanilla training SNNs, $T$ remains consistent between training and inference.In SFA training, the timestep during training is $T$, but it is expanded to $T\times D$ timesteps during inference (The neuron in the picture is IF-SR). Since the activations in SFA training are integers, another meaning of $D$ is the upper limit of the integer values allowed for activation during training. In the context of vanilla training SNNs, $D$ is equal to 1. (a) Direct training uses the same processing approach for both static and dynamic vision tasks. Since $D=1$, direct training will not be able to avoid the mechanism defects of binary spike firing in dynamic vision tasks. (b) For static tasks, the SFA method has advantages in training cost (one-timestep training) and inference energy cost (the image is input only once). (c) Fot dynamic tasks, the SFA method can simultaneously mitigate the flaws of spiking neurons in both spatial representation and temporal dynamics by expanding the timestep.}
\label{Fig_SFA}
\end{figure*}

\textbf{Multiple-timesteps SFA training for dynamic vision tasks.} For temporal dynamic vision tasks, multi-timestep simulation cannot be understood as rate encoding, but the same processing as static tasks is still used in direct training. Therefore, in dynamic vision tasks, the spatial representation and temporal dynamics defects of the binary spike firing will erupt simultaneously in direct training. 

SFA training overcomes these two flaws by expanding the timestep to $T \times D$ during inference: First, integer values reflect the importance of spatial inputs; Second, the size of the integers is determined by the spatial inputs, and thus is a data-dependent mechanism, equivalent to the fact that the spiking neurons can be dynamically forgetting (soft reset). In summary, each timestep in SFA training is expanded into $1\times D$ timesteps, which can be viewed as executing two reset mechanisms simultaneously. When the timestep is within the range of $[(t-1)\times D +1, t \times D]$, IF-SR neurons are used. When the timestep is at $[1,\cdots, (t-1)\times D, t\times D, \cdots, T\times D]$, LIF-HR/LIF-SR/IF-HR/IF-SR dynamic can be applied. We will discuss how to implement this hybrid spiking neuron mechanism in the neuromorphic chip in Section~\ref{sec_hardware_analysis}.

\subsection{Approximation Error Analysis}\label{sec:error_analysis}
Using finite integer values to approximate the continuous spike firing rate introduces an approximation error. The spike firing rate of the $l$-th layer over $D$ timesteps is:
\begin{equation}
    a_D^l = \frac{1}{D}\mathbf{S}^l_D = \frac{1}{D}\text{Fire}_D(\mathbf{U}^l) = \frac{1}{D}\lfloor\text{clip}\{\mathbf{U}^l, 0, D\}\rceil.
    \label{eq:spike_firing_rate_with_integer}
\end{equation}

\begin{definition}[Forward Approximation Error]
\label{Quan_error_one_timestep}
    Given the continuous input membrane potential at $l$-layer, denoted as $\mathbf{U}^l[d]$, we define the approximation error as follows:
\begin{equation}
    \text{Err}^{l} = \text{ReLU}(\mathbf{U}^l[d]) - \mathcal{F}(\mathbf{U}^l[d]),
\end{equation}
where $\text{ReLU}(\cdot)$ represents the ReLU function and $\mathcal{F}(\cdot)$ is the approximation function. In the ideal scenario $\mathcal{F}(\mathbf{U}^l[d])=\text{ReLU}(\mathbf{U}^l[d])$, the approximation error equals to zero.
\end{definition}

Since the approximation function in SFA is $\mathcal{F}(\cdot)=\frac{1}{D}\text{Fire}_D(\cdot)$, according to the Definition \ref{Quan_error_one_timestep}, the approximation error is :
\begin{equation}
    \text{Err}_{\text{SFA}}^{l} = \text{ReLU}(\mathbf{U}^l[d]) - \frac{1}{D}\lfloor\text{clip}\{\mathbf{U}^l[d], 0, D\}\rceil.
    \label{eq:Quan_error_one_timestep}
\end{equation}
Continuing with the following deductions:
\begin{equation}
 \text{Err}^{l}_{\text{SFA}} = \begin{cases} 0, & \text{if } \mathbf{U}^l[d] < 0 \\ \mathbf{U}^l[d] -\frac{1}{D}\lfloor \mathbf{U}^l[d] \rceil, & \text{if } 0\leq \mathbf{U}^l[d] <D \\ \mathbf{U}^l[d] - 1, & \text{if } \mathbf{U}^l[d] \geq D 
 \label{q_error_s}\end{cases} 
\end{equation}

In the backpropagation stage, the SFA (Eq.~\eqref{eq:integer_replace_spike_train}) is non-differentiable. Previous studies have introduced various surrogate gradient functions \cite{wu2018spatio,fang2023spikingjelly}, primarily designed to address binary spike outputs. In our approach, we consistently utilize rectangular windows as the surrogate function. For simplicity, we retain gradients only for neurons activated within the $[0, D]$ range, setting all others to zero.
\par

Using rectangular windows as the surrogate function cannot avoid gradient error. We define this gradient error as the difference between the integral of the surrogate gradient function and the forward approximation function of SFA, which can be expressed as follows:
\begin{definition}[Backward Gradient Error]
\label{grad_error}
Given the continuous input membrane potential at $l$-layer, denoted as $\mathbf{U}^l[d]$, we define the gradient  error as follows:
\begin{equation}
     \text{Err}^{l} =  \int_{-\infty}^{\infty} \text{Rect}_{[0,D]}(\mathbf{U}^l[d]) -\mathcal{F}(\mathbf{U}^l[d]),
     \label{grad}
\end{equation}
where $\text{Rect}_{[0,D]}(\cdot)$ is rectangular windows between 0 and $D$. $\mathcal{F}(\cdot)$ is the approximation function. In the ideal scenario, $\nabla \mathcal{F}(\mathbf{U}^l[d])=\text{Rect}_{[0,D]}(\mathbf{U}^l[d])$, the gradient error equals zero.
\end{definition}
According to the Definition \ref{grad_error}, the gradient error is:
\begin{equation}
    \text{Err}_{\text{SFA}}^{l} = \int_{-\infty}^{\infty} \text{Rect}_{[0,D]}(\mathbf{U}^l[d]) - \frac{1}{D}\lfloor\text{clip}\{\mathbf{U}^l[d], 0, D\}\rceil.
    \label{eq:grad_error_one_timestep}
\end{equation}
Continuing with the following deductions:
\begin{equation}
 \text{Err}^{l}_{\text{SFA}} = \begin{cases} 0, & \text{if } \mathbf{U}^l[d] < 0 \\ \mathbf{U}^l[d] -\frac{1}{D}\lfloor \mathbf{U}^l[d] \rceil, & \text{if } 0\leq\mathbf{U}^l[d] <D  \\ 
 1, & \text{if } \mathbf{U}^l[d] \geq D 
 \label{g_error_s}\end{cases} 
\end{equation}

In this subsection, we examine the approximation and gradient errors introduced by the $\text{Fire}_D(\cdot)$ function in SFA as depicted in Equations \eqref{q_error_s} and \eqref{g_error_s}. Our analysis reveals that larger values for the timestep $D$ lead to reduced errors. It should be noted that large timesteps may boost latency and energy consumption. 

\begin{table*}[!t]
\centering
\caption{Comparison with previous works on ImageNet-1k. The default inference input resolution is 224$\times$224. * The input crops are enlarged to 384$\times$384 in inference. We rewrite the timesteps of all direct training SNNs into the format of $T \times D$, where $T$ refers to the timestep and $D$ is the upper limit of the integer activation during training. In previous SNNs, $D$ defaulted to 1. Note, ``Param" and ``Step" in all Table headers of this paper denote ``Parameters" and ``Timestep", and all ``Step" involved in direct training SNNs are written in the form of $T \times D$ by default.}
\begin{tabular}{ccccccc}
\toprule
Methods & Architecture & Spike-driven & Param (M) & Power (mJ) & Step & Top-1 Acc.(\%) \\
\midrule
\multirow{10}{*}{ANN}  & DeiT-S\cite{touvron2021training}& \ding{55} & 22.0 & 21.2 & 1 & 79.8\\
& PVT-S\cite{wang2021pyramid} & \ding{55} & 24.5 & 35.0 & 1 & 79.8 \\
& Swin-T\cite{liu2021swin} & \ding{55} & 29.0 & 20.7 & 1 & 81.3 \\
& {{CAFormer-M36\cite{yu2023metaformer}}} &\ding{55} & 56.0 & 60.7 & 1 & 85.2 \\
& T2T-$\text{ViT}_t$-24\cite{yuan2021tokens} & \ding{55} & 64.1 & 69.0 & 1 & 82.6\\
& RepLKNet-31B\cite{ding2022scaling} & \ding{55} & 79.0 & 70.4 & 1 & 83.5 \\
& DeiT-B\cite{touvron2021training} & \ding{55} & 86.0 & 80.5 & 1 & 81.8 \\
& FocalNet-B (LRF)\cite{yang2021focal} & \ding{55} & 88.7 & 70.8 & 1 & 83.9 \\
& ConvNeXt-B\cite{liu2022convnet} & \ding{55} & 89.0 & 70.8 & 1 & 83.8 \\
& Focal-B\cite{srinivas2021bottleneck} & \ding{55} & 89.8 & 73.6 & 1 & 83.8 \\
& CrossFormer-L\cite{wang2023crossformer++} & \ding{55} & 92.0 & 74.1 & 1 & 84.0 \\
\midrule
\multirow{6}{*}{ANN2SNN} & ResNet-34\cite{Rathi2020Enabling} & \ding{51} & 21.8 & - & 250 & 61.5\\
& Optimal ANN2SNN\cite{bu2021optimal} & \ding{51} & - & - & 1024 & 74.3\\
 & Tandem VGG-16\cite{wu2021progressive} & \ding{51} & - & - & 16 & 73.0\\
& Fast-SNN\cite{10254579} & \ding{51} & 138.4 & - & 7 & 73.0\\
& Masked spiking Transformer\cite{wang2023masked} & \ding{55} & 28.5 & - & 512 & 78.5\\
& Two-stage ANN2SNN\cite{10361844} & \ding{51} & - & - & 128 & 74.9\\
\midrule
\multirow{11}{*}{\begin{tabular}[c]{@{}c@{}}Direct\\   training\end{tabular}} & SEW-ResNet\cite{fang2021deep} & \ding{55} & 60.2 & 12.9 & $4\times1$  & 69.2     \\
& TET-ResNet\cite{deng2022temporal}  & \ding{55}  & 21.8 & - & $4\times1$  &  68.0    \\
& DSR-ResNet\cite{meng2022training}  & \ding{55}  & 11.2 & - & $50\times1$  &  67.7    \\
& GAC-MS-ResNet\cite{qiu2024gated} &  \ding{51}  & 21.93 & 2.2 & $4\times1$ & 69.7  \\
& MS-ResNet\cite{hu2024advancing} &  \ding{51}  & 77.3 & 10.2 & $4\times1$ & 75.3  \\
 & Att-MS-ResNet\cite{yao_attention_Pami} & \ding{55} & 78.4 & 7.3 & $4\times1$ & 77.1 *     \\
 
\cline{2-7}
 & Spikformer\cite{zhou2023spikformer} & \ding{55} & 66.3 & 21.5 & $4\times1$  & 74.8     \\
 & Spike-dirven Transformer\cite{yao2023spike}  & \ding{51}   & 66.3 & 6.1 & $4\times1$  & 76.3     \\
 & Spikformer v2\cite{zhou2024spikformer} & \ding{55} & 172.7 & 25.6 & $4\times1$  & 82.4     \\  \cline{3-7}
& \multirow{3}{*}{Meta-SpikeFormer\cite{meta_spikeformer}} & \ding{51} & 15.1  & 16.7 & $4\times1$ & 74.1     \\
&   &  \ding{51}  & 31.3  & 32.8 & $4\times1$ & 77.2     \\
&  &  \ding{51}  & 55.4 & 52.4  & $4\times1$ & 80.0    \\ \hline

\multirow{11}{*}{\begin{tabular}[c]{@{}c@{}}Direct\\   training\end{tabular}} & \multirow{11}{*}{\begin{tabular}[c]{@{}c@{}}E-SpikeFormer\\ \textbf{(This Work)}\end{tabular}}  & \ding{51}  & 5.1 & 1.7 & $1\times4$ & 75.3 \\
 &  & \ding{51}  & 10.0 & 3.0 & $1\times4$ & 78.5 \\
 &  & \ding{51}   & 19.0  & 5.9 & $1\times4$ & 79.8 \\ 
  &  & \ding{51}  & 52.0 & 14.4  & $1\times4$ & 80.6 \\ \cline{3-7}
 &  & \ding{51}  & 83.0 & 19.1  & $1\times4$ & 83.2 \\
 &  & \ding{51}  & 83.0 & 30.9  & $1\times8$ & 84.0 \\
 &  & \ding{51}  & 83.0 &   -  & $1\times8$ & 85.2 * \\ \cline{3-7}
 &  &  \ding{51}  & 173.0  & 35.6 & $1\times4$ & 84.7 \\
 &  &  \ding{51}  & 173.0  & 54.7 & $1\times8$ & 85.1 \\
 &  &  \ding{51}  & 173.0  & - & $1\times8$ & \textbf{86.2} * \\
\bottomrule
\end{tabular}
\label{tab:image classify}
\end{table*}

\begin{figure}[!t]
\centering
\includegraphics[scale=0.32]{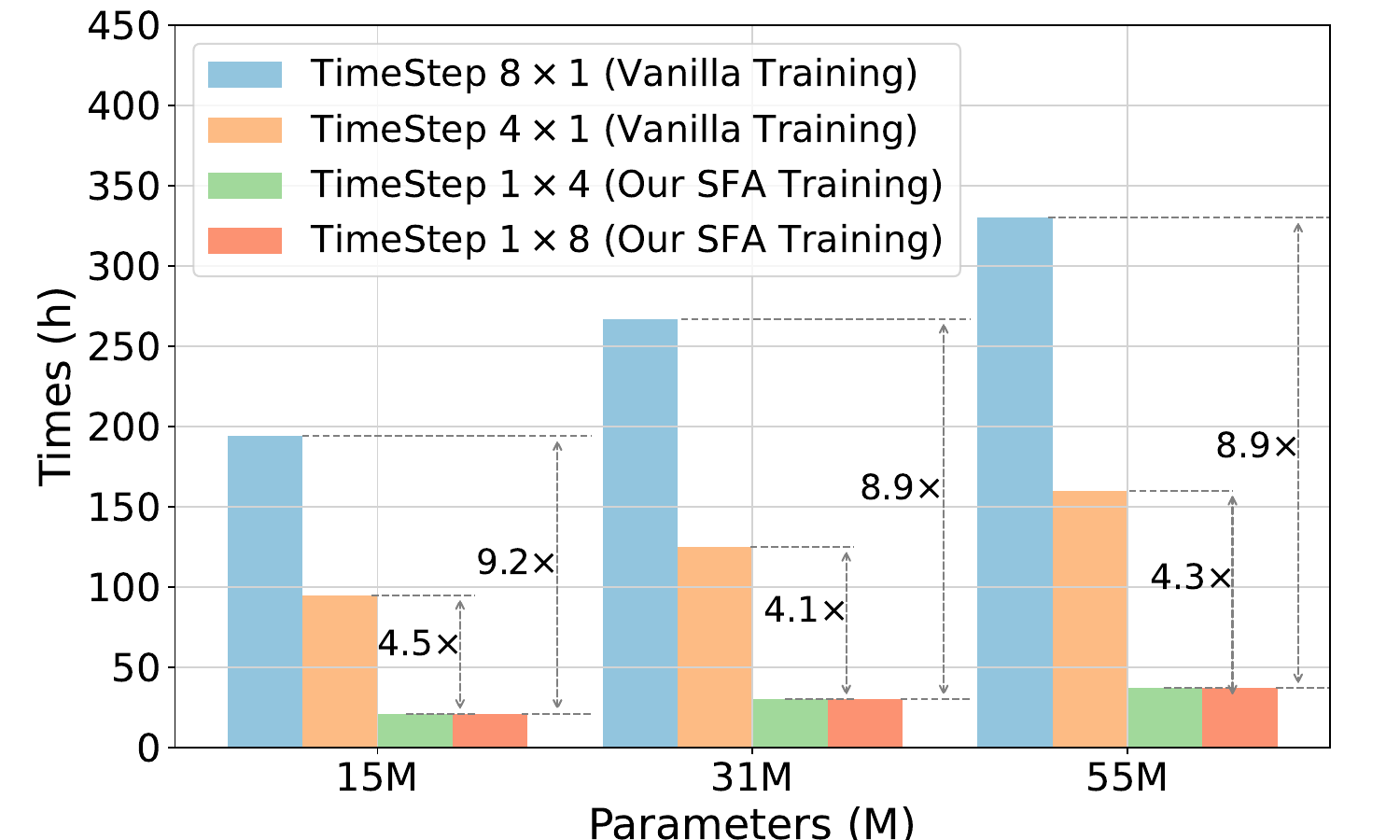}
\caption{Vanilla and SFA training time on the ImageNet with 200 epochs on 8 NVIDIA-A100-40G. On the static dataset, we use $T=1$ during training and extend it to $D$ timesteps during inference. Since we exploit one timestep training, the training cost is significantly reduced.}
\label{fig_speed_up}
\end{figure}


\section{Experiments}\label{section_Experiment}
In Section~\ref{subsec_main_reslut}, we report the main results of the proposed methods across various vision tasks. In Section~\ref{subsec_ablation_study}, we conduct ablation studies on SFA, E-SpikeFormer, and MIM pre-training. Sections~\ref{subsec_spike_pattern} and Sections~\ref{subsec_SFA_scale} analyze the impact of SFA in performance, power and scalability of SNNs. 

\subsection{Main Results}\label{subsec_main_reslut}
We evaluate our methods on various tasks, including image classification (ImageNet-1k \cite{deng2009imagenet}), object detection (COCO \cite{lin2014microsoft}), semantic segmentation (ADE20K \cite{zhou2017scene}), neuromorphic (HAR-DVS \cite{wang2022hardvs}, \rone{CIFAR10-DVS \cite{Li_Cifar10_DVS_2017}}) tasks. We follow the pre-training and fine-tuning approach in \cite{meta_spikeformer} to test E-SpikeFormer on COCO and ADE20K datasets. The model trained on ImageNet is used to fine-tune the detection or segmentation heads. Power evaluation methods and details of training/configurations are given in the Section~S1 and Section~S2 of Supplementary Materials, respectively.

\subsubsection{Image Classification}\label{subsubsec:image_classify}
The proposed methods are validated at 6 model scales (Table~\ref{tab:image classify}) on ImageNet-1K \cite{deng2009imagenet}, a prominent dataset with 1.3M training and 50K validation images across 1K classes. 

\textbf{Compared with prior SNN Baselines.} We comprehensively compare our E-SpikeFormer with prior SNNs in accuracy, parameter, and power. E-SpikeFormer achieves SOTA with clear all-around advantages, e.g., \textbf{E-SpikeFormer} vs. Meta-SpikeFormer \cite{meta_spikeformer} vs. MS-ResNet \cite{hu2024advancing}: Param, \textbf{19M} vs. 55M vs. 77M; Acc, \textbf{79.8\%} vs. 80.0\% vs. 76.3\%; Power, \textbf{5.9mJ} vs. 52.4mJ vs. 10.2mJ. We have brought \emph{qualitative} improvements to the SNNs' accuracy on ImageNet. In the small-scale 5M, 10M, and 19M models, we achieve 75.3\%, 78.5\%, and 79.8\%, respectively, far exceeding prior SNNs. We extend the size and performance upper limit of SNNs to 173M and \textbf{86.2\%}. The targeted efficient architecture design significantly reduces the power of our model. In Section~\ref{subsec_ablation_study}, we place the relevant ablation studies. $D$ in the SFA setting affects accuracy and power (Section~\ref{subsec_spike_pattern}) and scale (Section~\ref{subsec_SFA_scale}). Since we use one-timestep training on ImageNet, the training cost is also significantly reduced. As shown in Fig.~\ref{fig_speed_up}, we obtain a training speedup of about $9\times$ on various scales, and going from $1\times4$ to $1\times8$ does not result in an increase in training hours. \rone{In addition, we provide the GPU memory acceleration ratio of SFA method in Section~S3.}


\begin{table}[t]
\centering
\caption{Performance of detection on COCO val2017 \cite{lin2014microsoft}. * These methods exploit the pre-trained models on ImageNet as the backbone, then add detection heads for fine-tuning.}
\setlength{\tabcolsep}{1.5pt}
\begin{tabular}{cccccc}
\toprule
Methods & Architecture & \begin{tabular}[c]{@{}c@{}}Param\\      (M)\end{tabular} & \begin{tabular}[c]{@{}c@{}}Power\\      (mJ)\end{tabular} & Step & \begin{tabular}[c]{@{}c@{}}mAP\\      @50(\%)\end{tabular} \\
\midrule
\multirow{3}{*}{ANN}  & ResNet-18\cite{yu2022metaformer}  & 31.2 & 890.6 & 1  & 54.0     \\
& PVT-Tiny\cite{wang2021pyramid}  & 32.9 & 1040.5 & 1 & 56.7    \\
& DETR-DC5 \cite{zhu2020deformable} & 41.0 & 860.2 & 1 & 55.7\\
\midrule
\multirow{3}{*}{ANN2SNN}  & Spiking-Yolo\cite{kim2020spiking} & 10.2 & - & 3500 & 25.7     \\
& Spike Calibration\cite{li2022spike}    & 17.1 & - & 512 & 45.3 \\
& {{Fast-SNN \cite{hu2023fast}}}   & \rone{25.1} & \rone{-} & \rone{15} & \rone{46.4} \\
\midrule
\multirow{4}{*}{\begin{tabular}[c]{@{}c@{}}Direct\\ training\end{tabular}}  & Spiking Retina\cite{zhang2023direct}  & 11.3 & - & $4 \times 1$ & 28.5     \\
& EMS-Res-SNN\cite{su2023deep}   & 26.9 & - & $4 \times 1$ & 50.1    \\ \cline{3-6}
& \multirow{2}{*}{Meta-SpikeFormer * \cite{meta_spikeformer}}& 34.9 & 49.5 & $1 \times 1$ & 44.0     \\
&  & 75.0 & 140.8 & $1 \times 1$  & 51.2  \\
\midrule
\multirow{3}{*}{\begin{tabular}[c]{@{}c@{}}Direct\\ training\end{tabular}}
&\multirow{3}{*}{\begin{tabular}[c]{@{}c@{}}E-SpikeFormer *\\ \textbf{(This Work)}\end{tabular}} &   38.7 & 56.2& $1 \times 2$ & 41.8\\
&  &   38.7 & 94.5 & $1 \times 4$ & 58.4 \\
& & 38.7 & 119.5  & $1 \times 8$ & \textbf{58.8} \\
\bottomrule
\end{tabular}
\label{tab:object_detect}
\end{table}

\textbf{Compared with ANN Baselines.} The accuracy of E-SpikeFormer has been able to surpass or comparable to some classic ANNs while maintaining the energy efficiency advantage. For example, at the 80M scale, \textbf{E-SpikeFormer} vs. ConvNeXt-B \cite{liu2022convnet} vs. CrossFormer-L\cite{wang2023crossformer++}: Param, \textbf{83M} vs. 89M vs. 92M; Acc, \textbf{83.2\%} vs. 83.8\% vs. 84.0\%; Power, \textbf{19.1 mJ} vs. 70.8mJ vs. 74.1mJ. At the 20M scale, \textbf{E-SpikeFormer} vs. DeiT-S \cite{touvron2021training} vs. PVT-S\cite{wang2021pyramid}: Param, \textbf{19M} vs. 22M vs. 24.5M; Acc, \textbf{79.8\%} vs. 79.8\% vs. 79.8\%; Power, \textbf{5.9 mJ} vs. 21.2mJ vs. 35.0mJ. This is the first time that the SNN domain can compete with the performance of classic ANN-ViT models.

\subsubsection{Object Detection}\label{subsubsec:object_detect}
Before the emergence of Meta-SpikeFormer \cite{meta_spikeformer}, the SNN domain could only process object detection tasks through specific architectural designs, and the accuracy was low. Meta-SpikeFormer enable SNN to process dense prediction tasks such as object detection and semantic segmentation in a unified way. In this work, we test E-SpikeFormer on the COCO dataset\cite{lin2014microsoft}, which includes 118K training images (train2017) and 5K validation images (val2017). \rone{Results are shown in Table~\ref{tab:object_detect}, E-SpikeFormer achieves SOTA results in SNNs, and exceeds or is comparable to classical models in ANN but with a lower energy cost}. For example, \textbf{E-SpikeFormer} vs. Meta-SpikeFormer \cite{meta_spikeformer} vs. PVT-Tiny \cite{wang2021pyramid}: Param, \textbf{38.7M} vs. 75.0M vs. vs. 32.9M; mAP@0.5, \textbf{58.8\%} vs. 51.2\% vs. 56.7\%; Power, \textbf{119.5mJ} vs. 140.8mJ vs. 1040.5mJ.

\begin{table}[t]
\centering
\caption{Performance of segmentation on ADE20K\cite{zhou2017scene}. * These methods use the pre-trained models on ImageNet as the backbone, then add segmentation heads for fine-tuning.}
\setlength{\tabcolsep}{3.5pt}
\begin{tabular}{cccccc}
\toprule
Methods & Architecture & \begin{tabular}[c]{@{}c@{}}Param\\      (M)\end{tabular} & \begin{tabular}[c]{@{}c@{}}Power\\      (mJ)\end{tabular} & Step & MIoU (\%) \\
\midrule
\multirow{6}{*}{ANN}  & ResNet-18\cite{yu2022metaformer}   & 15.5 & 147.1 & 1 & 32.9     \\
& PVT-Tiny\cite{wang2021pyramid}    & 17.0 & 152.7 & 1 & 35.7    \\
& PVT-Small\cite{wang2021pyramid}    & 28.2 & 204.7 & 1 & 39.8    \\
& InternImage-T\cite{wang2023internimage} & 59.0 & 2171.2 & 1 & 48.1 \\
& ConvNeXt-T\cite{liu2022convnet} & 60.0 & 2173.5 & 1 & 45.8\\
& DeepLab-V3\cite{zhang2022resnest}   & 68.1 & 1240.6 & 1 & 42.7    \\
\midrule
\multirow{4}{*}{\begin{tabular}[c]{@{}c@{}}Direct\\ training\end{tabular}} & \multirow{4}{*}{\begin{tabular}[c]{@{}c@{}}Meta-\\ SpikeFormer *\cite{meta_spikeformer}\end{tabular}}   & 16.5 & 22.1 & $1 \times 1$ & 32.3     \\
&   & 16.5 & 88.1 & $4 \times 1$  & 33.6    \\
&   & 59.8 & 46.6 & $1 \times 1$  & 34.8    \\
&  & 59.8 & 183.6 & $4 \times 1$  & 35.3    \\
\midrule
\multirow{3}{*}{\begin{tabular}[c]{@{}c@{}}Direct\\ training\end{tabular}} &  \multirow{3}{*}{\begin{tabular}[c]{@{}c@{}}E-SpikeFormer *\\ \textbf{(This Work)}\end{tabular}}
&  11.0 & 18.4 & $1 \times 2$ & 31.9 \\
&  & 11.0 & 27.2& $1 \times 4$ &40.1\\
& &  11.0 & 33.6 & $1 \times 8$ & 41.4 \\
\bottomrule
\end{tabular}
\label{tab:semantic seg}
\end{table}

\begin{table}[t]
\caption{Performance of neuromorphic human action recognition on HAR-DVS \cite{wang2022hardvs}.}
\begin{center}
\setlength{\tabcolsep}{3.5pt}
\begin{tabular}{cccccc}
\toprule
\multicolumn{1}{c}{ Methods} &\multicolumn{1}{c}{Architecture}
&\begin{tabular}[c]{@{}c@{}}Param\\      (M)\end{tabular}
& \begin{tabular}[c]{@{}c@{}}Power\\      (mJ)\end{tabular} & Step 
&\multicolumn{1}{c}{ Acc.(\%)}\\
\midrule
\multirow{3}{*}{ANN}  & SlowFast \cite{feichtenhofer2019slowfast} & 33.6&- &$8 \times 1$& 46.5 \\
   & ACTION-Net \cite{wang2021action}   & 27.9&- & $8 \times 1$  & 46.9 \\
    & TimeSformer \cite{bertasius2021space}   & 121.2 &- & $8 \times 1$ & 50.8 \\
\midrule
\multirow{3}{*}{\begin{tabular}[c]{@{}c@{}}Direct\\ training\end{tabular}}  & Res-SNN-34 \cite{fang2021deep}   & 21.8&- &  $8 \times 1$  & 46.1\\
& ASA-SNN \cite{yao2023inherent}   & 41.5 &-&  $8 \times 1$ & 47.1\\
& Meta-SpikeFormer\cite{meta_spikeformer}  & 18.3 & 8.0&$8 \times 1$ & 47.5\\
\midrule
\multirow{2}{*}{\begin{tabular}[c]{@{}c@{}}Direct\\ training\end{tabular}}  & E-SpikeFormer  & 18.7 &18.1 & $8 \times 4$& 48.9\\
& (\textbf{This Work})& 18.7 &23.5 &$8 \times 8$ & \textbf{49.2}\\
\hline
\end{tabular}
\end{center}
\label{table_har_dvs}
\end{table}

\begin{table}[t]
    \centering
    \caption{Ablation studies of SFA and architecture design. }
    \tabcolsep=0.15cm
    \begin{tabular}{ccccc}
    \toprule
        Methods  & Param (M) & Step & Power (mJ) & Acc.(\%)  \\
        \midrule
        \multirow{2}{*}{Baseline \cite{meta_spikeformer}}   & 10.0 & 1 $\times$ 1 & 3.3 & 68.5 \\
         & 10.0 & 4 $\times$ 1& \textbf{11.9} & \textbf{71.3} \\ \hline
        + Our SFA training & 10.0 & 1 $\times$ 4 & 7.3 \textbf{(-4.6)} & 77.8 \textbf{(+6.5)} \\
        + Our architecture & 10.0 & 1 $\times$ 4 & 3.0 \textbf{(-4.3)} &  78.5 \textbf{(+0.7)}\\
        \bottomrule
    \end{tabular}
    \label{tab:ablation_sfa}
\end{table}

\subsubsection{Semantic Segmentation}\label{subsubsec:semantic_seg}
ADE20K \cite{zhou2017scene} is a challenging semantic segmentation benchmark, including 20K and 2K images in the training and validation set, respectively, and covering 150 categories. Meta-SpikeFormer \cite{meta_spikeformer} reported the first results of SNN on the ADE20K dataset. We follow the experimental setting in Meta-SpikeFormer to test our E-SpikeFormer on ADE20K. Our model performs well on ADE20K. For example, \textbf{E-SpikeFormer} vs. Meta-SpikeFormer \cite{meta_spikeformer}: Param, \textbf{11.0M} vs. 59.8M ; MIoU, \textbf{41.4\%} vs. 35.3\%; Power, \textbf{33.6mJ} vs. 183.6mJ. 


\subsubsection{Nuromorphic Human Action Recognition}
HAR-DVS \cite{wang2022hardvs} is the largest event-based Human Activity Recognition (HAR) dataset currently, containing 300 classes and 107,646 samples, acquired by a DAVIS346 camera with a spatial resolution of $346\times260$. We use HAR-DVS to evaluate E-SpikeFormer, data preprocessing and experimental design (please see Section~S2.2) follow the previous Meta-SpikeFormer. We handle HAR-DVS in a direct training manner with $T=8$. Meta-SpikeFormer achieves comparable accuracy to ANNs and is better than prior SNNs (Table~\ref{table_har_dvs}). \rone{We also verify the effectiveness of our methods on neuromorphic CIFAR10-DVS, and the experimental results are given in the supplementary materials (Section~S4).}

\subsection{Ablation Study}\label{subsec_ablation_study}
In Table~\ref{tab:ablation_sfa}, we conduct rigorous ablation experiments to verify the gains of SFA training and E-SpikeFormer compared to the baseline (Meta-SpikeFormer \cite{meta_spikeformer}) on ImageNet. 

\romannumeral1) Baseline vs. \textbf{Baseline+SFA}: Step, $4 \times 1$ vs. $1 \times 4$; Power, 11.9mJ vs. \textbf{7.3mJ(-4.6mJ)}; Acc, 71.3\% vs. \textbf{77.8\%(+6.5\%)}. This implies that the strategy of SFA is far more effective and efficient than existing directly training SNN.  

\romannumeral2) We advance the SNN architecture, Baseline+SFA vs. \textbf{E-SpikeFormer+SFA}: Power, 7.3mJ vs. \textbf{3.7mJ(-4.3mJ)}; Acc, 77.8\% vs. \textbf{78.5\%(+0.7\%)}. The proposed efficient architecture design further improves accuracy while reducing power. 

In addition, we observe the effect of SFA method on energy cost, and conduct in depth analysis in Section~\ref{subsec_spike_pattern}. Together, by combining SFA and E-SpikeFormer, compared with the 10M baseline \cite{meta_spikeformer}, we simultaneously achieve a \textbf{7.2\%} performance improvement, \textbf{5.3}$\times$ training time acceleration (12 hours vs. 63 hours on 8 NVIDIA-A100-40G with 200 epoch), and \textbf{3.7}$\times$ inference energy efficiency. 

\rone{Due to space limitations, ablation experiments on MIM training are placed in Section~S5 of the supplementary materials, including MIM mask ratio, decoder width and depth, and studies on MIM and distillation.}

\begin{figure*}[!ht]
\centering
\includegraphics[scale=0.65]{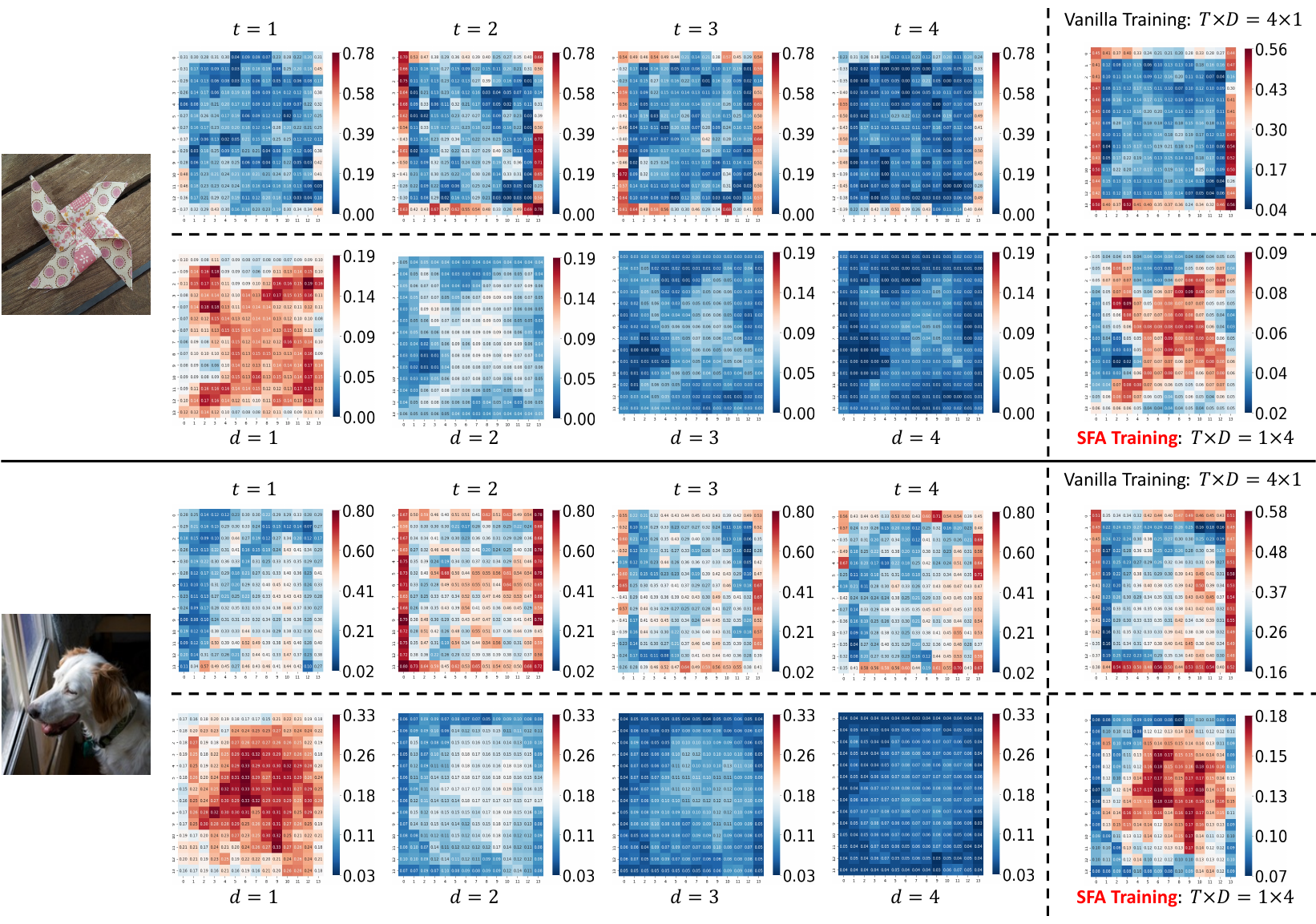}
\caption{Visualization of spike distribution in vanilla and SFA training on ImageNet. From top to bottom: vanilla and SFA training. For a single sample, the output spike maps of a certain SNN layer in the form of the 4D ($[T, C, H, W]$). We first averaged the 4D spike maps into a 3D map ($[T, H, W]$) over the channel dimension. Since $T=4$, we can split the 3D map into four 2D maps (the four maps in the middle of this figure). Then, we averaged the 4D spike maps into a 2D map ($[H, W]$) over the temporal and channel dimension, resulting in the two rightmost spike maps. For a single map, each pixel represents the firing rate of a spiking neuron. The redder the pixel, the higher the spike firing rate; the bluer the pixel, the closer the spike firing rate is to 0. The four spike maps in the middle show the differences in the temporal dimension between the spike distributions obtained by vanilla and SFA training: spike firing rate does not change much from left to right in vanilla training (first row); spike firing rate decreases from left to right in SFA training (second row). The two spike maps on the far right show that the spike features obtained by SFA training are clearer, and the firing rate is lower.}
\label{Fig_spike_map}
\end{figure*}




\subsection{Advantages of SFA in Power and Performance}\label{subsec_spike_pattern}

\begin{figure}[t]
    \centering
    \includegraphics[scale=0.23]{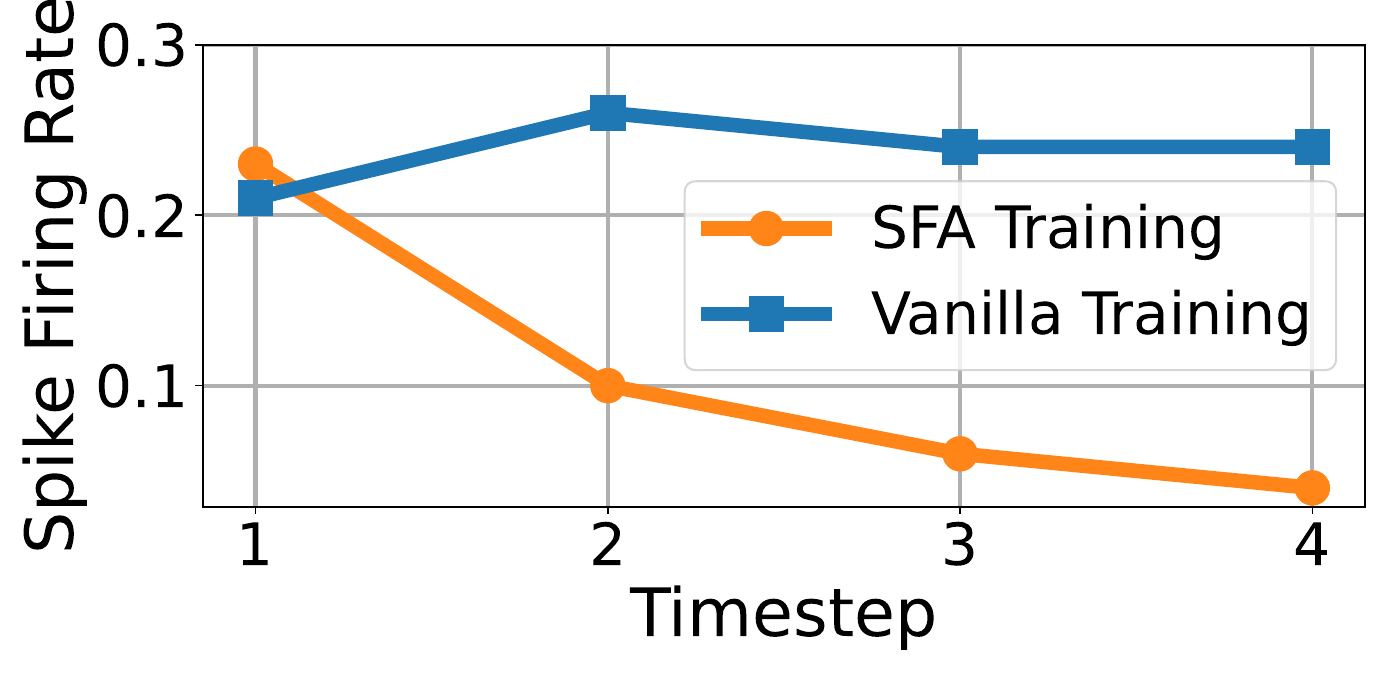}
    \caption{Network Spike Firing Rate (NSFR) of vanilla ($T \times D = 4 \times 1$) training and SFA ($T \times D = 1 \times 4$) training at each timestep. Experimental settings are the same as the ablation studies in Table~\ref{tab:ablation_sfa}. The sum of NSFR over all timesteps implies the average spike firing profile over the entire dataset: in vanilla training and SFA training, each spiking neuron fires an average of 0.96 and 0.39 spikes, respectively.}
    \label{Fig_SFR}
\end{figure}

Compared with ANN2SNN and direct training, the SFA method implies a novel spike firing pattern (Fig.~\ref{Fig_spike_firing_pattern}). We here focus on the following question: ``How does the SFA method optimize the power and performance of SNNs?"

We primarily examine the spike distribution under different spike firing patterns. First, we define the \emph{Network Spike Firing Rate (NSFR)}: at timestep $t$, SNN's NSFR is the ratio of spikes produced over all the neurons to the total number of neurons in this timestep. NSFR is used to evaluate the change of the same network's spike distribution at different timesteps. By default, we count spike firing rates of the samples in the test set and get their mean as the final data. Second, we \emph{visualizes the spike maps} of the last layer of the last block of the 10M model in Fig.~\ref{Fig_spike_map}. The output spike maps of a certain SNN layer in the form of the 4D ($[T, C, H, W]$), which can be averaged into a 3D map ($[T, H, W]$) over the channel dimension. Since $T=4$, we can split the 3D map into four 2D maps, which are the four maps in the middle of Fig.~\ref{Fig_spike_map}. Each pixel of spike map represents the firing rate of a spiking neuron. The redder the pixel, the higher spike firing rate; the bluer the pixel, the closer the spike firing rate is to zero. Then we have the following observations:


\romannumeral1) \textbf{Power optimization.} Fig.~\ref{Fig_SFR} shows the NSFRs of vanilla and SFA training SNNs. In vanilla training, NSFR at each timestep is close (the blue line). In contrast, the NSFR of the SNN trained by SFA method will decrease as the timestep increases (the orange line). Fig.~\ref{Fig_spike_map} shows the spike distributions for the two samples obtained using the vanilla and SFA training. It can be seen that the color of the spike map obtained by SFA training (second row of each sample) changes from red to blue from left to right, which means that as the timestep increases, the spike firing rate decreases. In contrast, the spike map obtained by vanilla training (first row of each) does not change much in color. This phenomenon coincides with the data in Fig.~\ref{Fig_SFR}. And, the highest spike firing rates of spiking neurons (the value of the reddest pixel in the figure) obtained by SFA and vanilla training are 0.19 and 0.78, respectively.

\romannumeral2) \textbf{Performance optimization.} In Fig.~\ref{Fig_spike_map}, the 4D spike maps can also be averaged into a 2D map ($[H, W]$) over the temporal and channel dimension, resulting in the two rightmost spike maps. It can be observed that compared with vanilla training, the spike feature obtained by SFA training are clearer and the firing rate is lower.

\begin{figure}[t]
\centering
\subfigure[MIM training]{\label{fig_degrad}\includegraphics[scale=0.13]{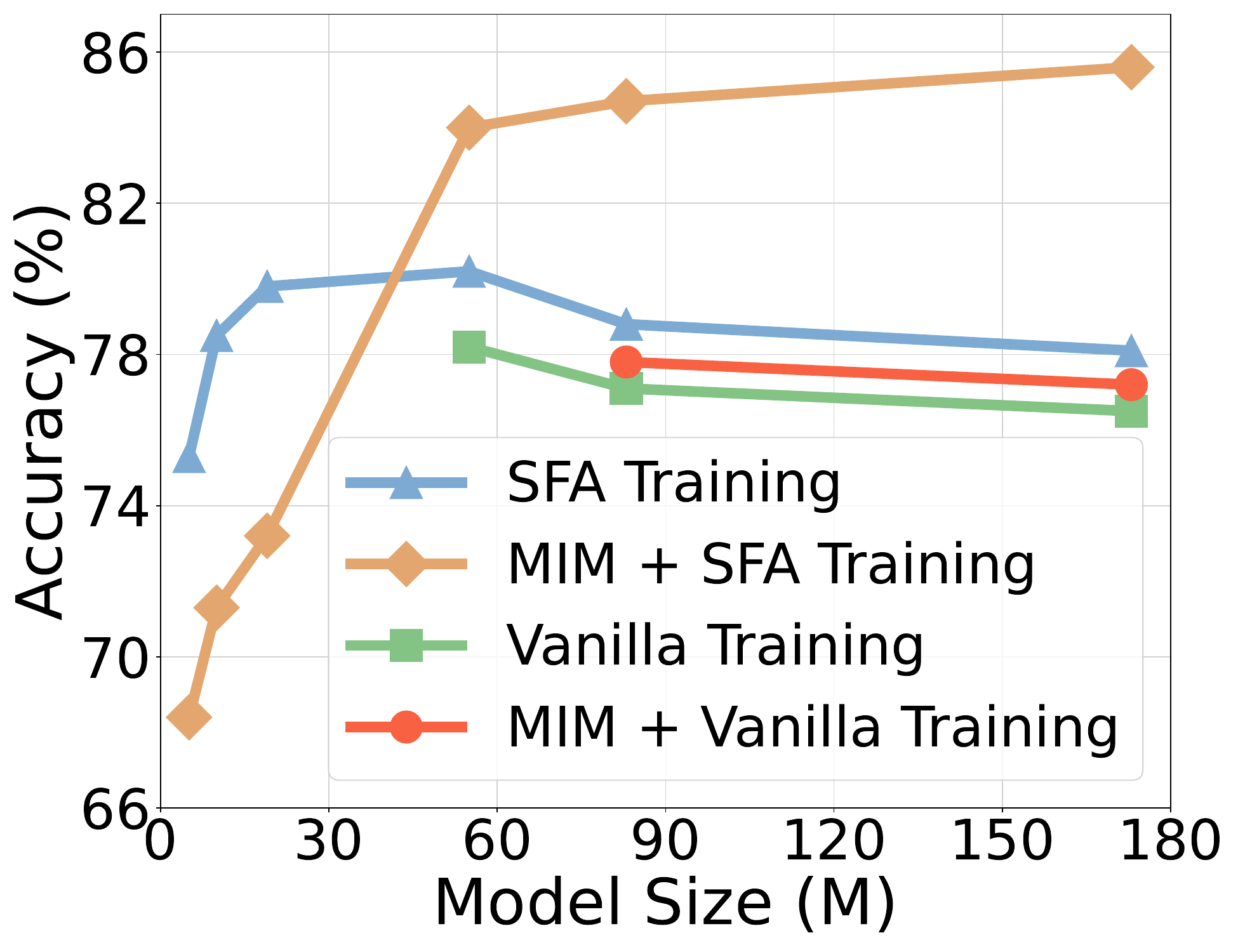}}
\subfigure[Effective rank analysis]{\label{fig_rank} \includegraphics[scale=0.13]{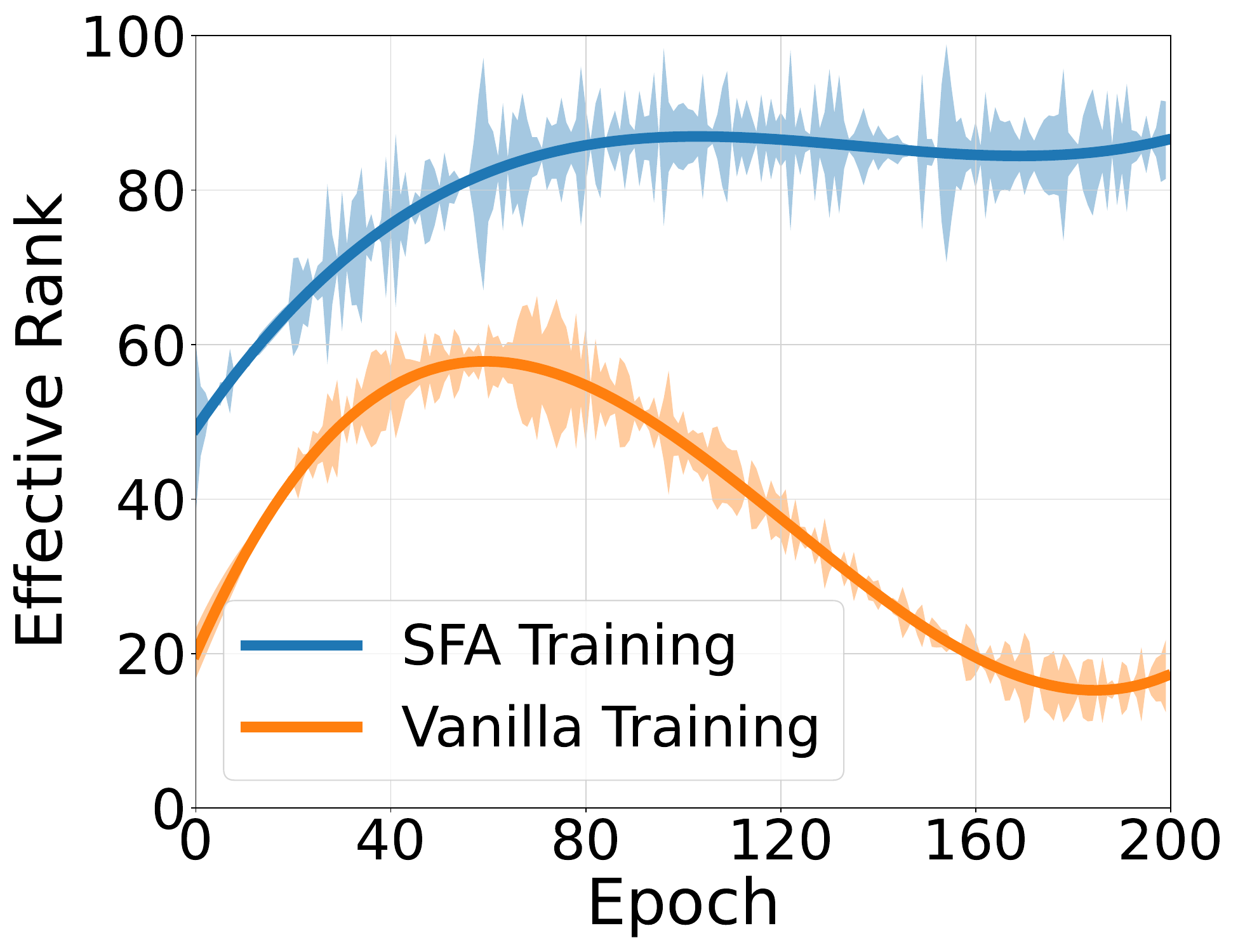}}
\caption{Binary spike firing interferes with the scaling of SNNs. (a) Ablation experiments on MIM pre-training and SFA training. In direct training with binary spike firing, performance degradation still occurs exploiting MIM pre-training. By contrast, in SFA training using integer activations, MIM pretraining effectively scales up SNNs. (b) Experimental results on effective rank \cite{roy2007effective} in MIM-SNN.}
\end{figure}

\subsection{Binary Spike Firing Mechanism is NOT Conducive to the Scaling of SNNs}\label{subsec_SFA_scale}

We analyze why directly scaling E-SpikeFormer results in performance degradation from the perspective of encoder design in Section~\ref{subsection_scale_training}. The issue arises due to vanilla spike convolution leaking information after the mask operation in MIM is superimposed. Spike sparse convolution is proposed to address this. Moreover, we observe that binary spike firing also interferes with the scaling of SNNs. As shown in Fig.~\ref{fig_degrad}, we analyze the scaling of SNNs from two orthogonal perspectives: with and without MIM-SNN pre-training and SFA versus direct training. As shown by the orange line, the MIM strategy can effectively help SNNs scaling when trained using the SFA method. However, in the vanilla training with binary spikes, performance degradation still occurs after using the MIM strategy (red line).

We argue that the potential reason is that the binary spikes affect the feature extraction ability of the encoder part in MIM-SNN. Specifically, the feature extraction capability can be measured by the effective rank~\cite{zhang2022mask,roy2007effective}:  
\begin{equation}
    \text{Rank}(\mathbf{A}) = \exp(-\sum\limits_{k=1}^{n} p_k \log(p_k)),
    \label{rank}
\end{equation}
where $p_k = \frac{\sigma_k}{\Vert \sigma \Vert_1}$ is the distribution of singular values. $\sigma_1,\dots,\sigma_n$ is the singular value on the diagonal matrix $\Sigma$. It can be obtained by the SNN encoder's output $\mathbf{Z}$, where $\mathbf{Z}=\mathbf{U}\Sigma \mathbf{H}^{T}$. A high effective rank indicates that rich feature information is being extracted, while a low effective rank suggests the opposite. 

Experimental results are given in Fig.~\ref{fig_rank}. In MIM-SNN with vanilla training, the effective rank decreases over training (Fig. \ref{fig_rank}), suggesting collapsing feature representation into a lower-dimensional space, potentially harming generalization and task performance. This could lead to difficulty distinguishing input samples, a lack of complexity capture, and performance degradation. By contrast, in MIM-SNN with SFA, the effective rank increases and tends to be stable during training, indicating robust representations and better fine-tuning performance.

\section{Discussion and Summary}\label{sec_discussion}

This section discusses the way in which the SFA method and MIM-based pre-trained SNNs can be executed on neuromorphic chips. And, we summarize the profound effects of changes in spike firing patterns induced by the SFA method.

\subsection{Implementation of SFA on Neuromorphic Chip}\label{sec_hardware_analysis}

\begin{figure}[t]
\centering
\includegraphics[scale=0.95]{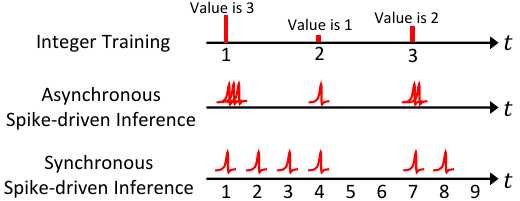}
\caption{Possible implementations of spike-driven inference on neuromorphic chips. In the figure, it is assumed that $T=D=3$. Asynchronous spike-driven has no global clock~\cite{Speck}, and the integer $D$ is converted into $D$ spikes fired in a very short time window during inference. Synchronous spike-driven expands the $T$ timesteps in integer training to $T \times D$ timesteps during inference.}
\label{Fig_chip_imp}
\end{figure}

The spike firing patterns for vanilla and SFA training are different (Fig.~\ref{Fig_spike_firing_pattern}). In ANN2SNN and vanilla direct training, the spikes at multi-timesteps can be considered random. By contrast, the spikes in SFA training are concentrated and continuous. This is crucial and implies that the mechanisms of implementation of these two methods on neuromorphic chips are different. 

Fig.~\ref{Fig_chip_imp} shows in what way the proposed SFA method can be executed on a neuromorphic chip:

\romannumeral1) \textbf{Synchronous implementation.} In this case, integers are synchronously converted to spikes. Specifically, since synchronous neuromorphic chips contain a global clock, $T$ timesteps are re-expanded into $T \times D$ timesteps. In static visual tasks, only the timesteps during inference have changed. In dynamic visual tasks, different spiking neurons are required at different timesteps. Tianjic theoretically supports these operations~\cite{Nature_1}.

\romannumeral2) \textbf{Asynchronous implementation.} In this case, the implementation of ``a single timestep is expanded $D$ times" on the chip is ``the asynchronous neuromorphic chip emits a series of spikes (up to $D$) in a very short time window." Asynchronous neuromorphic chips do not have a global clock, so there is actually no concept of timestep (please refer to the ``Synchronous Training and Asynchronous Deployment" part in \cite{Speck} for more details), and the spike firing between different spiking neurons is independent. When there is an event input, a part of spiking neurons on the chip will be activated, and the rest of the time they are in a idle state. And, there is no need to switch the working mode of the spiking neuron in order to extend the additional timesteps. Speck~\cite{Speck} supports this operation.

We make the following discussion about the pros and cons of implementing SFA methods synchronously or asynchronously:

\romannumeral1) The SFA method is implemented by either synchronous or asynchronous neuromorphic chips. However, vanilla training SNNs are difficult to implement with asynchronous chips. As shown in Fig.~\ref{Fig_spike_firing_pattern}, vanilla training SNNs must execute all timesteps to reach the spike firing approximation, thus requiring a globally synchronized clock. By contrast, SFA training SNNs only need to complete the fire of a consecutive string of $D$ spikes, and the spiking neurons can be rested. No global clock is required.

\romannumeral2) When $D$ is small, as well as the spike firing rate is small, the asynchronous spike-driven implementation of SFA method certainly has outstanding advantages in terms of power and latency. The risk point is that when $D$ is not limited, the communication bandwidth of the chip may be exceeded in the condition of a large number of spikes being fired in a short period of time. But this risk is controllable. 

Therefore, considering the unique characteristics of the SFA method, it is evident that it is more compatible with asynchronous implementation, which has significant implications for the design and development of algorithm-driven neuromorphic chips. \rone{Another interesting and noteworthy point is that for the SFA method, an increase in $D$ may lead to higher energy consumption, but the extent of this increase is closely related to the specific task. For instance, adjusting $1\times4$ to $1\times8$ results in a significantly higher increase in energy cost in image classification (Table~\ref{tab:image classify}) than in detection (Table~\ref{tab:object_detect}) and segmentation (Table~\ref{tab:semantic seg}) tasks.}

\begin{figure}[t]
    \centering
    \includegraphics[width=1\linewidth]{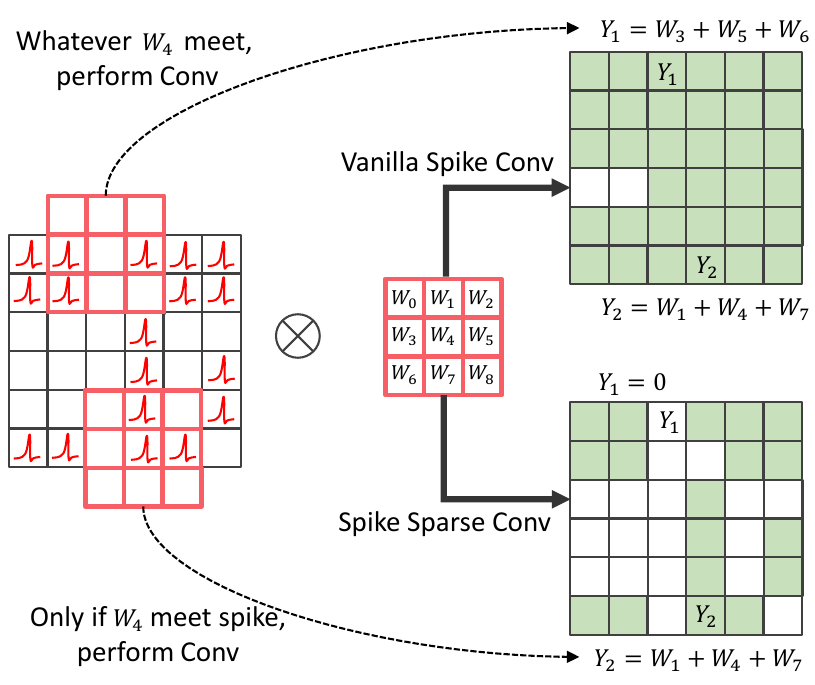}
    \caption{Comparison of Spike Sparse Conv (SSC) and  Vanilla Spike Conv (VSC). On a neuromorphic chip, when a spike occurs, the address mapping function finds the synapses and neurons that need to be added, and then takes out the corresponding weights to perform the addition operations. The only difference between VSC and SSC is the addressing mapping function. In SSC, it is specified that the convolution is performed only if there is a spike input at the position corresponding to $W_4$ (the center position of the convolution kernel). VSC does not have this restriction.}
    \label{fig:sparseconv}
\end{figure}

\subsection{Neuromorphic Chips Naturally Support MIM-SNN Pre-training}\label{subsec_scale_2}

Spike convolution is performed on neuromorphic chips in a spike-driven manner. Event-by-event sparse processing enables neuromorphic computing to enjoy the advantages of low latency and low power consumption. The system state is changed based on each spike input. Upon receiving a spike, the SNN core identifies the corresponding weight and target neuron position via address search and updates the target neuron state accordingly~\cite{Speck}. 

Spike Sparse Convolution (SSC) is proposed to solve the information leakage problem during MIM-SNN pre-training. Since the Vanilla Spike Convolution (VSC) itself is a sparse computation convolution, we are interested in the difference and connection between VSC and SSC. As shown in Fig. \ref{fig:sparseconv}, we observe that SSC actually only has one more judgment condition than VSC when performing spike convolution. In SSC, the convolution is performed only if there is a spike input at the position corresponding to $W_4$ (the center position of the convolution kernel). Therefore, SSC can always pass the information in the spike feature map to the corresponding position of the next layer. This is important in cases where a MIM encoder is required, as most of the information needs to be masked. \rone{Moreover, masked images are first divided into patches, and SSC records the positions of these masks in E-SpikeFormer’s MIM pretraining. Then SSC operates exclusively at unmasked locations where the central kernel firing spikes.}

Given the commonality of VSC and SSC, we only need to slightly modify the addressing mapping function corresponding to VSC to execute SSC on the neuromorphic chip. In other words, current neuromorphic chips can support on-chip learning based on MIM pre-training.

\subsection{Summary}\label{subsec_SFA_summary}

After discussing the hardware implementation, we summarize here the effect of spike firing patterns on SNNs. Various spike firing patters, e.g., $\{\mathbf{S}^l[d]\}_D$ in ANN2SNN and vanilla training versus $\{\hat{\mathbf{S}}^l[d]\}_D$ in our SFA training, result in varying spike distributions. These differences subsequently influence the training, power, performance, scaling, and hardware implementation of SNNs. All three training methods can be thought of as rate coding. We still use Fig.~\ref{Fig_spike_firing_pattern} as an example of the difference between them. ANN2SNN and direct training must run all ten timesteps to complete the rate coding. SFA training, on the other hand, only needs to fire spikes at the first few timesteps. This minor yet significant design improvement has profound impacts:

\romannumeral1) \textbf{Reducing training costs (Fig.~\ref{fig_sfa_training} and Fig.~\ref{fig_speed_up}).} When dealing with static vision tasks, SFA requires only one-timestep training, avoiding the huge training memory and time cost with traditional mulit-timestep direct training.

\romannumeral2) \textbf{Less power cost (Fig.~\ref{Fig_spike_map} and Fig.~\ref{Fig_SFR}).} In the static task, the inference stage of the SFA method only uses the image input once, without the need to repeat the input inference as in the traditional direct training. On the other hand, the spike distribution obtained by the SFA method has an spike firing rate that decreases progressively with timesteps.

\romannumeral3) \textbf{Higher performance (Fig.~\ref{Fig_spike_map}).} The SFA method yields a clearer spike map.

\romannumeral4) \textbf{Easy to scale (Fig.~\ref{fig_rank}).} Integer training is more likely to learn good features.

\romannumeral5) \textbf{More suitable for execution on asynchronous neuromorphic chips (Fig.~\ref{Fig_chip_imp}).} The inference process in the SFA method is naturally suited to be deployed on asynchronous neuromorphic chips, which can fully embody the low-power and low-latency features of SNNs.

\section{Conclusion}\label{section:conclusion}
We address the performance and training consumption gap between SNNs and ANNs. A key contribution is identifying the mechanistic flaw of binary spike firing in spiking neurons. To overcome these limitations, we propose a Spike Firing Approximation (SFA) method. This method is based on integer training and spike-driven inference, aiming to optimize the spike firing pattern of spiking neurons. Our results demonstrate that optimization the spike firing pattern leads to comprehensive improvements in SNNs, including enhanced training efficiency, reduced power consumption, improved performance, easier scalability, and better utilization of neuromorphic chips. Additionally, we develop an efficient spike-driven Transformer architecture and a spike masked autoencoder to prevent performance degradation during SNN scaling. By addressing the training and performance challenges of large-scale SNNs, we pave the way for a new era in neuromorphic computing.

\ifCLASSOPTIONcompsoc
\section*{Acknowledgments}
This work was partially supported by National Science Foundation for Distinguished Young Scholars (62325603), National Natural Science Foundation of China (62236009, U22A20103, 62441606).

\else
\section*{Acknowledgment}
\fi


\bibliographystyle{IEEETran}
\bibliography{./ref}

\appendices

\ifCLASSOPTIONcaptionsoff
  \newpage
\fi
\end{document}